\documentclass{article} %
\usepackage{iclr2027_conference,times}
\usepackage[T1]{fontenc}
\usepackage{microtype}
\usepackage{amsmath,amssymb,amsthm}

\newtheorem{theorem}{Theorem}[section]
\newtheorem{proposition}[theorem]{Proposition}
\newtheorem{corollary}[theorem]{Corollary}
\usepackage{enumitem}

\usepackage[hidelinks]{hyperref}
\usepackage{amsmath,amsfonts,bm}

\def\eqref#1{equation~\ref{#1}}
\def\1{\bm{1}}

\DeclareMathAlphabet{\mathsfit}{\encodingdefault}{\sfdefault}{m}{sl}
\SetMathAlphabet{\mathsfit}{bold}{\encodingdefault}{\sfdefault}{bx}{n}

\usepackage{hyperref}
\usepackage{url}

\usepackage{graphicx}       %
\usepackage{booktabs}       %
\usepackage{bbding}         %
\usepackage{makecell}       %
\usepackage{multirow}       %
\usepackage{diagbox}        %
\usepackage{array}          %
\usepackage[table]{xcolor}  %
\usepackage{capt-of}        %
\usepackage{longtable}      %

\definecolor{OPDBlue}{HTML}{3E78A0}
\definecolor{OPDLightBlue}{HTML}{E7F2FA}
\definecolor{TableHeader}{HTML}{F5F7F9}
\newcommand{\bst}[1]{\textbf{#1}}
\newcommand{\subbst}[1]{\underline{#1}}
\newcommand{\caseaction}[4][white]{\cellcolor{#1}#2 & \cellcolor{#1}{\sffamily #3} & \cellcolor{#1}#4}

\title{From Imitation to Reward Discovery:\\On-Policy Warmup for Agentic RL}

\author{{\fontsize{9.5}{11.5}\selectfont\bfseries Yitong Qiao$^{1,2,}$\thanks{Work done during Yitong's internship at Ant Group.}\hspace{0.2em},
Tiantian He$^{2}$,
Lei Liu$^{1,2,}$\thanks{Corresponding authors: Lei Liu and Zhixuan Chu.}\hspace{0.2em},
Yue Shen$^{2}$,
Jian Wang$^{2}$,
Jinjie Gu$^{2}$,
Zhixuan Chu$^{1,\dagger}$}\\
{\normalfont\normalsize $^1$Zhejiang University \quad $^2$Ant Healthcare, Ant Group}\\
{\normalfont\normalsize
\href{mailto:qiaoyt@zju.edu.cn}{\textcolor{blue!55!black}{\texttt{qiaoyt@zju.edu.cn}}};
\href{mailto:liulei1497@gmail.com}{\textcolor{blue!55!black}{\texttt{liulei1497@gmail.com}}};
\href{mailto:zhixuanchu@zju.edu.cn}{\textcolor{blue!55!black}{\texttt{zhixuanchu@zju.edu.cn}}}}
}

\iclrfinalcopy
\hypersetup{
  pdftitle={From Imitation to Reward Discovery: On-Policy Warmup for Agentic RL},
  pdfauthor={Yitong Qiao; Tiantian He; Lei Liu; Yue Shen; Jian Wang; Jinjie Gu; Zhixuan Chu}
}
\usepackage{eso-pic}
\begin{document}
\AddToShipoutPictureFG*{%
  \AtPageUpperLeft{%
    \put(\LenToUnit{108bp},\LenToUnit{-78bp}){%
      \includegraphics[height=17bp]{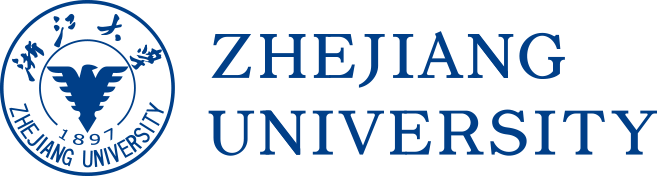}\hspace{18bp}%
      \includegraphics[height=17bp]{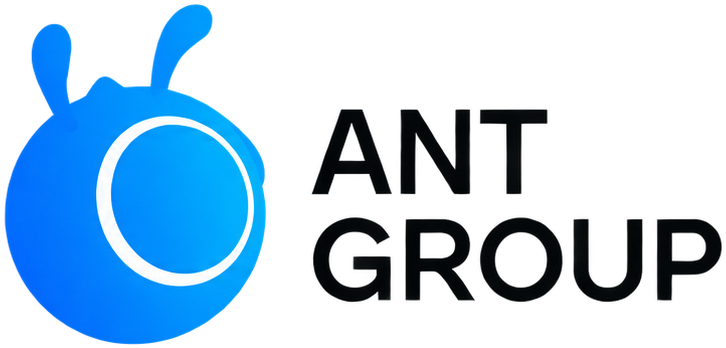}%
    }%
  }%
}
\raggedbottom
\addtocontents{toc}{\protect\setcounter{tocdepth}{-1}}
{\setlength{\tabcolsep}{4.5pt}\maketitle}
\lhead{\parbox{\textwidth}{\footnotesize From Imitation to Reward Discovery: On-Policy Warmup for Agentic RL}}
\setlength{\headheight}{22pt}

\begin{abstract}
Reinforcement learning with a verifiable reward (RLVR) offers a scalable approach to training language-model agents, yet sparse outcome rewards can leave early training with little signal for policy improvement. We identify an \textbf{On-Policy Acceleration Phenomenon}: in our main comparisons, RLVR initialized with on-policy distillation reaches high performance earlier in training and achieves both higher average performance during subsequent RLVR and higher final performance than the alternative baselines. Motivated by this observation, we study \textbf{On-Policy Warmup (OPW)}, a teacher-guided stage in which the student trains with teacher supervision on its own interaction trajectories before transitioning to RLVR. Unlike imitation on fixed teacher-generated trajectories, OPW targets states induced by the student’s own decisions, including imperfect actions and recovery situations. We provide a theoretical explanation by connecting on-policy reverse-KL distillation to trajectory-level distribution matching. Under a competent teacher and sufficiently small population distillation loss, this connection yields a lower bound on initial verifier success and a corresponding bound on reward-discovery complexity. For group-relative RLVR, we further characterize when increased success probability produces more reward-informative groups. Together, our findings support on-policy distillation as an effective warmup for agentic RLVR and identify initial reward discovery as a mechanism that can contribute to the observed acceleration.
\end{abstract}

\begin{figure}[!htbp]
\centering
\includegraphics[width=0.949\textwidth]{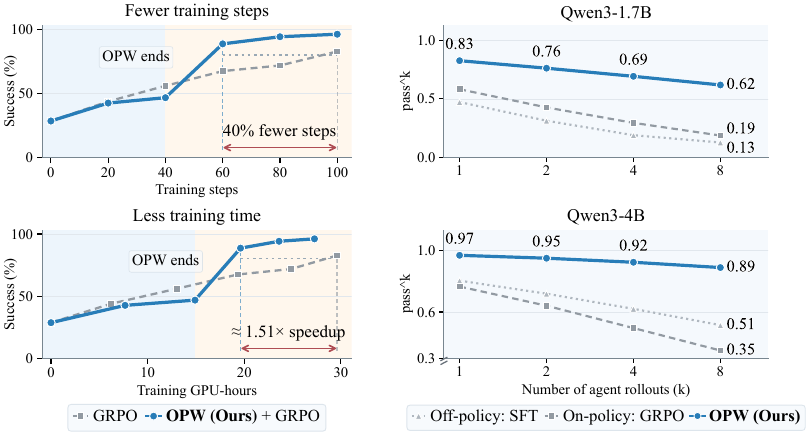}
\caption{\textbf{OPW accelerates GRPO and improves repeated task success in ALFWorld.}
\textbf{Left:}~Qwen3-4B Seen learning curves for GRPO with and without OPW. GPU-hours include teacher scoring and exclude evaluation.
\textbf{Right:}~Final Unseen repeated task success (pass\textasciicircum{}k).}
\label{fig:teaser}
\end{figure}

\section{Introduction}
\label{sec:introduction}

Reinforcement learning with verifiable rewards (RLVR) improves language-model reasoning using automatically checkable outcomes rather than human annotations \citep{shao2024deepseekmath,deepseekai2025r1}. It is appealing for agents that interleave reasoning with tool use and environment feedback \citep{yao2023react}. Yet outcome-based feedback offers limited guidance about the intermediate decisions required for success. This creates a cold-start challenge: when success requires several coordinated decisions, an initial policy may produce mostly failures with identical rewards. Under binary outcome rewards, all-failure groups provide no within-group task-reward contrast for group-relative methods such as GRPO \citep{shao2024deepseekmath}. Effective early training requires not just diverse trajectories, but successes frequent enough for the verifier to distinguish useful behavior.

Teacher supervision offers a way to improve this starting point, where knowledge distillation can transfer a denser supervision signal to a student \citep{hinton2015distilling}. However, imitation on fixed teacher-generated trajectories primarily supervises states visited by the teacher, whereas an agent must act on histories induced by its own decisions. This distribution mismatch is a central concern in imitation learning \citep{ross2011reduction} and is particularly consequential in interactive tasks: an imperfect action can alter subsequent tool outputs, available information, and opportunities for recovery. On-policy distillation addresses the mismatch by applying teacher supervision to student-generated trajectories \citep{agarwal2024gkd}. Whether such supervision provides a useful initialization for subsequent RLVR, rather than merely improving imitation, remains a distinct question.

In this work, we identify an \textbf{On-Policy Acceleration Phenomenon}: after on-policy distillation warmup, RLVR reaches high performance earlier and achieves higher final performance than the compared baselines. Motivated by this observation, we propose \textbf{On-Policy Warmup (OPW)}, a teacher-guided initialization stage for agentic RLVR. During warmup, the student generates trajectories by interacting with the environment and receives supervision from a fixed teacher. The student learns from this supervision before transitioning to RLVR, where training uses verifiable task outcomes without further teacher queries. Teacher supervision is confined to warmup.

The purpose of warmup is not simply to reduce distillation loss. The downstream objective is task success under sparse verifier feedback, and a student that imitates the teacher more closely need not, in general, learn faster during RLVR. Instead, we ask whether on-policy supervision makes successful trajectories accessible often enough to improve initial reward discovery \citep{chen2025coverage}. This question connects the distribution of trajectories induced by the warmed-up policy to the training signal available to the subsequent RLVR stage. It also separates the role of teacher guidance during initialization from the role of the verifier during later policy improvement.

We analyze this connection theoretically. Under shared environment dynamics, the cumulative reverse-KL distillation loss on student-visited states equals the KL divergence between student and teacher trajectory distributions. The data-processing inequality then bounds the gap between their success probabilities under the same binary verifier used to assess task completion. A sufficiently successful teacher and sufficiently small student population loss imply a positive lower bound on initial student success probability, and thus an upper bound on the expected rollouts to discover a success. For group-relative RLVR, we also characterize when higher success probability increases the likelihood of sampling a group containing both successes and failures. These results suggest initial reward discovery as a mechanism for the observed acceleration, but neither guarantee convergence nor imply lower total compute once warmup and teacher queries are included.

Our contributions are threefold:
\begin{itemize}[leftmargin=*]
    \item
    \textbf{We identify the On-Policy Acceleration Phenomenon in agentic RLVR.}
    In our main comparisons, on-policy distillation warmup leads to earlier high performance, higher average performance during subsequent RLVR, and higher final performance than the compared baselines.

    \item
    \textbf{We study OPW as a teacher-guided initialization stage for RLVR.}
    OPW provides teacher supervision along student-generated trajectories, including after unsuccessful actions, and is followed by RLVR without further teacher queries.

    \item
    \textbf{We connect on-policy reverse-KL distillation to initial reward discovery.}
    Through trajectory-level distribution matching, we derive a lower bound on initial verifier success and bound reward-discovery complexity under sufficient teacher competence and small population distillation loss.
\end{itemize}

\section{Related Work}

\begin{table}[t]
\centering
\begin{minipage}{0.94\linewidth}
\centering
\caption{Related work on warmup for RLVR.}
\label{tab:related_work}
\begingroup
\definecolor{RWGreen}{HTML}{009B55}
\definecolor{RWPurple}{cmyk}{0.75,0.90,0,0}
\newcommand{\rwYes}{\textcolor{RWGreen}{\CheckmarkBold}}
\newcommand{\rwNo}{\textcolor{red}{\XSolidBrush}}
\fontsize{7.5}{8}\selectfont
\setlength{\tabcolsep}{2pt}
\renewcommand{\arraystretch}{0.98}
\setlength{\aboverulesep}{0.25ex}
\setlength{\belowrulesep}{0.35ex}
\resizebox{\linewidth}{!}{%
\begin{tabular}{@{\hspace{8pt}}>{\columncolor{white}[8pt][\tabcolsep]}l|cc|cc|c>{\columncolor{white}[\tabcolsep][8pt]}c@{\hspace{8pt}}}
\toprule
& \multicolumn{2}{c|}{\textbf{Setting}}
& \multicolumn{2}{c|}{\textbf{Evidence}}
& \multicolumn{2}{c}{\textbf{Training study}} \\
\cmidrule(lr){2-3}\cmidrule(lr){4-5}\cmidrule(lr){6-7}
\textbf{Work}
& \makecell{On-policy\\warmup}
& \makecell{Environment\\interaction}
& \makecell{Theoretical\\analysis}
& \makecell{Action-level\\interventions}
& \makecell{Training\\cost}
& \makecell{Warmup\\duration} \\
\midrule
Sequential Beats Joint~\citep{li2026sequential}
& \rwYes & \rwNo & \rwNo & \rwNo & \rwNo & \rwYes \\
RL Starts before RL~\citep{dong2026rlstarts}
& \rwYes & \rwNo & \rwNo & \rwNo & \rwNo & \rwNo \\
OPDSearch+~\citep{ye2026opdsearch}
& \rwYes & \rwYes & \rwYes & \rwNo & \rwYes & \rwNo \\
PRISM~\citep{wang2026prism}
& \rwYes & \rwNo & \rwNo & \rwNo & \rwNo & \rwNo \\
PEAR~\citep{zhang2026pear}
& \rwNo & \rwNo & \rwYes & \rwNo & \rwNo & \rwNo \\
\midrule
\rowcolor{RWPurple!6}
\textbf{OPW (Ours)}
& \rwYes & \rwYes & \rwYes & \rwYes & \rwYes & \rwYes \\
\bottomrule
\end{tabular}%
}
\endgroup
\end{minipage}
\end{table}

\noindent\textbf{Warmup for RLVR.}
RLVR trains language-model policies using verifiable outcome rewards~\citep{shao2024deepseekmath}.
SFT is commonly used as warmup before RLVR in mathematical reasoning and logic tasks~\citep{luong2024reft,shrestha2025warmup}.
Stronger SFT performance does not necessarily lead to better performance after RL~\citep{kang2026quagmires,zhang2026pear}.
PEAR reweights fixed offline supervision using importance sampling to improve subsequent RL~\citep{zhang2026pear}.
TailSFT filters sequences according to their likelihood improvement during SFT to improve coverage and subsequent GRPO performance~\citep{malladi2026tailsft}.
For on-policy distillation (OPD), Sequential Beats Joint compares sequential and joint OPD--RLVR training and examines when to switch to RL~\citep{li2026sequential}.
RL Starts before RL shows that initial accuracy and pre-RL pass@$k$ do not fully explain subsequent performance, and compares trajectory sources and divergence objectives~\citep{dong2026rlstarts}.
PRISM inserts a black-box, response-level adversarial alignment stage between SFT and multimodal RLVR~\citep{wang2026prism}.
Closest to our interactive setting, OPDSearch+ applies teacher supervision to student trajectories collected through live search interactions before RL refinement~\citep{ye2026opdsearch}.
Table~\ref{tab:related_work} compares these warmup methods.

\noindent\textbf{On-Policy Supervision for Interactive Agents.}
Knowledge distillation transfers teacher behavior to a student~\citep{hinton2015distilling}.
For language models and agents, a common approach is SFT on teacher-generated responses or interaction trajectories~\citep{deepseekai2025r1,zeng2023agenttuning}.
GKD instead applies teacher feedback to student-generated sequences, addressing the mismatch between training data and the student's own outputs~\citep{agarwal2024gkd}.
MiniLLM studies reverse-KL distillation to discourage the student from assigning excessive probability to regions that are unlikely under the teacher~\citep{gu2024minillm}.
Guided-OPD extends this line of work to multi-turn agents by mixing teacher- and student-generated turns and gradually reducing the probability of teacher intervention~\citep{li2026guidedopd}.
Other methods combine ongoing RL with teacher trajectories~\citep{yan2025luffy}, token-level guidance from interaction feedback~\citep{wang2026openclawrl}, or self-distillation from completed interactions~\citep{wu2026seed}.
We study teacher supervision on student-generated trajectories as warmup before teacher-free RLVR in multi-turn environments.
Under population reverse-KL assumptions, we connect trajectory distributions to initial verifier success and reward-discovery complexity.
We remove supervision at repeated invalid actions, evaluate teacher-guided choices through task continuations from common histories, and track the retention of teacher-target preferences during GRPO.
We also compare warmup durations and recorded training costs, including teacher supervision, to distinguish faster subsequent learning from lower total cost.

\section{Method}
\label{sec:method}

We study OPD as warmup for RLVR. Our analysis connects distillation to trajectory-level distribution matching and informative reward observations during early RLVR.

\subsection{Theoretical Motivation for On-Policy Warmup}
\label{sec:theory}

\noindent\textbf{Setup.}
Fix a task $x$ and a finite-horizon interaction process with horizon $H$. Let $s_t$ denote the full interaction history,
including previous actions and environment observations, and let
$a_t$ denote the next policy action. Variable-length trajectories
can be padded with an absorbing state and a shared dummy action, which contributes zero KL. For token-level language policies, $t$ indexes policy-generated tokens, including reasoning and action tokens, and $H$ bounds their number.
The student policy $\pi$ and teacher policy $\pi_T$ interact with
the same environment and initial-state distribution. Let
$P_\pi^x$ and $P_T^x$ denote their induced trajectory distributions.
For a binary verifier $R_x(\tau)\in\{0,1\}$ indicating complete task success, define
$p_\pi(x)
=
\mathbb{E}_{\tau\sim P_\pi^x}[R_x(\tau)]$ and $p_T(x)
=
\mathbb{E}_{\tau\sim P_T^x}[R_x(\tau)]$.

All theoretical statements below are conditional on $x$; we suppress this
dependence when unambiguous. We define the population on-policy distillation loss as
\begin{equation}
\mathcal{L}_{\mathrm{OPD}}(\pi)
=
\sum_{t=1}^{H}
\mathbb{E}_{s_t\sim d_t^\pi}
\left[
D_{\mathrm{KL}}
\bigl(
\pi(\cdot\mid s_t)
\,\|\,
\pi_T(\cdot\mid s_t)
\bigr)
\right],
\label{eq:opd-loss}
\end{equation}
where $d_t^\pi$ is the student's state distribution.
We assume finite population loss and use natural logs.

\begin{proposition}[Trajectory-Level Interpretation]
\label{prop:trajectory-kl}
Under the shared-environment assumption,
$D_{\mathrm{KL}}(P_\pi\|P_T)=\mathcal{L}_{\mathrm{OPD}}(\pi)$.
\end{proposition}

\begin{theorem}[Success Transfer from Teacher to Student]
\label{thm:success-transfer}
If $\mathcal{L}_{\mathrm{OPD}}(\pi)\leq\varepsilon$, then $\operatorname{KL}(p_\pi\|p_T)\leq \varepsilon,$ where $\operatorname{KL}$ denotes the KL divergence between
Bernoulli distributions. In particular,
\begin{equation}
|p_\pi-p_T|
\leq
\sqrt{\frac{\varepsilon}{2}},
\qquad
p_\pi
\geq
\left[
p_T-\sqrt{\frac{\varepsilon}{2}}
\right]_+.
\label{eq:success-bound}
\end{equation}
\end{theorem}

Theorem~\ref{thm:success-transfer} provides a sufficient condition for transferring verifier success from a competent teacher. When the loss is reported as an average over $H$ steps, $\overline{\mathcal{L}}_{\mathrm{OPD}} =\mathcal{L}_{\mathrm{OPD}}/H$, the corresponding deviation bound is $\sqrt{H\overline{\mathcal{L}}_{\mathrm{OPD}}/2}$. A small average token-level loss alone need not give a tight guarantee over long interactions, because loss accumulates across action positions.

\noindent\textbf{Why On-Policy States Matter.}
The trajectory identity in
Proposition~\ref{prop:trajectory-kl} specifically requires
distillation under the student's distribution.
To illustrate the role of state coverage, define
$k_\pi(s)= D_{\mathrm{KL}}
\bigl(\pi(\cdot\mid s)\|\pi_T(\cdot\mid s)\bigr)
$
and consider evaluating the same local reverse-KL discrepancy under an
offline state distribution $\mu_t$: $\mathcal{L}_{\mu}(\pi)=
\sum_{t=1}^{H}\mathbb{E}_{s\sim\mu_t}[k_\pi(s)].$ If $d_t^\pi$ is absolutely continuous with respect to $\mu_t$ and $\frac{d d_t^\pi}{d\mu_t}(s)\leq C$ for every $t$ and almost every $s$, then
$\mathcal{L}_{\mathrm{OPD}}(\pi)\leq C\,\mathcal{L}_{\mu}(\pi).$ The success-transfer bound obtained from this offline reverse-KL loss thus depends on an additional state-coverage coefficient $C$. If $d_t^\pi\not\ll\mu_t$ for any $t$, no finite $C$ exists. OPD directly targets the state-weighted discrepancy appearing in
the trajectory KL, avoiding this additional distribution-transfer requirement at the population level. This is particularly relevant for agentic tasks, where imperfect actions can lead to novel tool outputs, execution failures, and recovery states.

\subsection{Implications for Early RLVR Reward Discovery}

For the student $\pi_{\mathrm{w}}$ after OPW, Theorem~\ref{thm:success-transfer} gives
$p_{\mathrm{w}}=p_{\pi_{\mathrm{w}}}\geq\ell=\left[p_T-\sqrt{\mathcal{L}_{\mathrm{OPD}}(\pi_{\mathrm{w}})/2}\right]_+$.

\begin{corollary}[Initial Reward-Discovery Complexity]
\label{cor:reward-discovery}
Suppose $\ell>0$. For independent rollouts from the fixed warm-start
policy $\pi_{\mathrm{w}}$, the probability of observing at least one
successful trajectory among $N$ rollouts satisfies
\begin{equation}
\Pr(\text{at least one success})=1-(1-p_{\mathrm{w}})^N
\geq
1-(1-\ell)^N
\geq
1-e^{-N\ell}.
\end{equation}
For any $\delta\in(0,1)$, $N\geq\left\lceil\frac{\log(1/\delta)}{\ell}\right\rceil$ is sufficient to observe a success with probability at least $1-\delta$. The expected number of rollouts until the first success is at most $1/\ell$ under the same fixed policy.
\end{corollary}

If the pre-warmup success probability is $p_0$ and $\ell>p_0$, OPD is certified to improve this initial discovery process. For group-relative methods with binary rewards, the relevant event is observing successes and failures in the same group. For $G\geq2$ independent rollouts on the
same task, its probability is $m_G(p)=1-p^G-(1-p)^G$.
Such groups contain nonzero empirical reward variance and can support reward-based differentiation among trajectories. The derivative
$m_G'(p)=G\bigl((1-p)^{G-1}-p^{G-1}\bigr)$
is nonnegative on $[0,1/2]$. Therefore, in the sparse-success regime $p_0<p_{\mathrm{w}}\leq1/2$, increasing success probability increases the frequency of reward-informative groups. The lower bound $p_{\mathrm{w}}\geq\ell$ alone does not imply an increase in reward-informative groups. This monotonicity does not hold globally: near $p=1$, groups can become uninformative because all trajectories succeed. For partial-credit rewards, unsuccessful trajectories can still receive different scores.

\begin{table*}[t]
\caption{Final task performance across three training paths after 100 steps (Qwen3-4B).}
\label{tab:continuous-objectives-final}
\centering\fontsize{7.5}{9}\selectfont
\renewcommand{\arraystretch}{1.12}
\setlength{\tabcolsep}{1pt}
\arrayrulecolor{black}
\begingroup
\begin{tabular}{@{\hspace{3pt}}>{\raggedright\arraybackslash}p{31pt}>{\raggedright\arraybackslash}p{78pt}*{7}{>{\centering\arraybackslash}p{17.8pt}}|*{7}{>{\centering\arraybackslash}p{17.8pt}}@{\hspace{3pt}}}
\toprule
\multirow{2}{*}{\textbf{Split}} & \multirow{2}{78pt}{\diagbox[width=78pt,height=20pt]{\textbf{Training}}{\textbf{Task type}}} & \multicolumn{7}{c|}{\textbf{ALFWorld}} & \multicolumn{7}{c}{\textbf{ScienceWorld}} \\
\cmidrule(lr){3-9}\cmidrule(l){10-16}
& & Pick & Clean & Heat & Cool & Exam & Pick2 & All & Find & Life & Gene. & Mix & Prop. & Phys. & All \\
\midrule
\multirow{3}{*}{\textbf{Seen}} & {\fontsize{7.5}{9}\selectfont Direct OPD} & 88.6 & 30.1 & 9.4 & 12.5 & 51.9 & 57.3 & 45.9 & 61.4 & \subbst{46.2} & 29.6 & \bst{30.4} & 52.1 & 44.2 & 48.9 \\
& {\fontsize{7.5}{9}\selectfont Direct GRPO} & \bst{100.0} & \subbst{83.3} & \subbst{78.9} & \subbst{83.5} & \subbst{85.6} & \subbst{95.8} & \subbst{89.4} & \subbst{80.3} & 39.8 & \subbst{55.3} & \subbst{25.1} & \subbst{61.1} & \subbst{62.1} & \subbst{61.2} \\
\rowcolor{OPDLightBlue}
\cellcolor{white} & {\fontsize{7.5}{9}\selectfont \textbf{OPW (Ours)} + GRPO} & \subbst{98.9} & \bst{98.1} & \bst{86.7} & \bst{93.5} & \bst{97.1} & \bst{97.9} & \bst{96.1} & \bst{97.0} & \bst{53.6} & \bst{78.7} & \bst{30.4} & \bst{98.0} & \bst{82.4} & \bst{87.3} \\
\midrule
\multirow{3}{*}{\textbf{Unseen}} & {\fontsize{7.5}{9}\selectfont Direct OPD} & \subbst{90.6} & 54.0 & 12.0 & 48.2 & 56.9 & 59.6 & 53.5 & 64.8 & 45.4 & 47.7 & 8.1 & 60.8 & 47.1 & 55.2 \\
& {\fontsize{7.5}{9}\selectfont Direct GRPO} & 88.5 & \subbst{82.3} & \subbst{83.7} & \subbst{90.5} & \subbst{83.3} & \subbst{94.1} & \subbst{86.6} & \bst{74.0} & \subbst{50.7} & \subbst{50.4} & \subbst{13.6} & \subbst{63.8} & \subbst{61.0} & \subbst{61.2} \\
\rowcolor{OPDLightBlue}
\cellcolor{white} & {\fontsize{7.5}{9}\selectfont \textbf{OPW (Ours)} + GRPO} & \bst{93.8} & \bst{100.0} & \bst{92.4} & \bst{99.4} & \bst{97.2} & \bst{97.1} & \bst{96.7} & \subbst{71.8} & \bst{62.5} & \bst{52.4} & \bst{15.4} & \bst{75.8} & \bst{74.9} & \bst{68.2} \\
\bottomrule
\end{tabular}
\endgroup
\end{table*}

\subsection{On-Policy Warmup for RLVR}
\label{sec:opw}

Motivated by this connection, we propose \emph{On-Policy Warmup} (OPW): a teacher-guided initialization stage preceding RLVR. Starting from the base model, the pipeline is $\pi_{\theta_0}
\;\xrightarrow{\text{OPD}}\;
\pi_{\theta_{\mathrm{w}}}
\;\xrightarrow{\text{RLVR}}\;
\pi_{\theta_{\mathrm{final}}}.$

\noindent\textbf{Stage I: Teacher Supervision on Student Trajectories.} Let $\mathcal{D}_{\mathrm{w}}$ and $\mathcal{D}_{\mathrm{RL}}$ denote the warmup and RLVR task distributions, and $\pi_T$ a fixed teacher policy. Our experiments use disjoint training task pools. At warmup iteration $k$, we sample a task $x\sim\mathcal{D}_{\mathrm{w}}$ and generate an interaction trajectory using the current student $\tau\sim P_{\pi_{\theta_k}}^x.$ We query the teacher at interaction histories along these trajectories. Unlike offline imitation of teacher trajectories, OPW exposes the teacher to states induced by the student's own decisions, including imperfect intermediate solutions, failed tool calls, and recovery.

The population reverse-KL objective considered in our analysis is
$
\mathcal{L}_{\mathrm{OPW}}(\theta)
=
\mathbb{E}_{x\sim\mathcal{D}_{\mathrm{w}}}
\left[
\mathcal{L}_{\mathrm{OPD}}(\pi_\theta;x)
\right]$, using the task-specific loss in Eq.~\ref{eq:opd-loss}. For this reverse-KL formulation, the corresponding fixed-history surrogate at iteration $k$ is
\begin{equation}
\widetilde{\mathcal{L}}_k(\theta)
=
\mathbb{E}_{
x\sim\mathcal{D}_{\mathrm{w}},
\,\tau\sim P_{\pi_{\theta_k}}^x
}
\left[
\sum_{t=1}^{H}
D_{\mathrm{KL}}
\left(
\pi_\theta(\cdot\mid s_t)
\,\|\,
\pi_T(\cdot\mid s_t)
\right)
\right].
\label{eq:opw-surrogate}
\end{equation}
Sampled histories are held fixed during each step, and trajectories are resampled as the student changes. Teacher supervision is applied
to policy-generated tokens, not to environment observations.

\noindent\textbf{Stage II: Verifier-Based Reinforcement Learning.}
After a prescribed warmup duration, we initialize RLVR with
$\pi_{\theta_{\mathrm{w}}}$ and optimize
\begin{equation}
J_{\mathrm{VR}}(\theta)
=
\mathbb{E}_{
x\sim\mathcal{D}_{\mathrm{RL}},
\,\tau\sim P_{\pi_\theta}^x
}
\left[r_x(\tau)\right],
\qquad
\theta\leftarrow\theta_{\mathrm{w}},
\label{eq:opw-rlvr}
\end{equation}
where $r_x(\tau)\in[0,1]$ is the training reward. It equals the binary verifier $R_x$ in ALFWorld and includes partial progress in ScienceWorld. We use GRPO's group-relative updates without probability-ratio clipping or additional KL or entropy terms (Table~\ref{tab:training-settings}). In the basic OPW pipeline, the distillation loss is removed after warmup and no further teacher queries are required.

\noindent\textbf{Scope.}
The bounds concern initial reward discovery under task-specific population reverse-KL control, with success measured under the same policy distributions. The analysis does not establish that finite warmup meets this condition or that it transfers across tasks or decoding settings.

\section{Experiments}
\label{sec:comparison}

\subsection{Experimental Setup}

\noindent\textbf{Benchmarks.}
We evaluate agents on ALFWorld~\citep{shridhar2021alfworld} and ScienceWorld~\citep{wang2022scienceworld}, which require multi-step interaction to complete household and scientific tasks. ALFWorld covers placing, cleaning, heating, cooling, examining objects under a light, and placing two objects, with 3,553 training tasks and 140 Seen and 134 Unseen test tasks. ScienceWorld contains 2,294 training tasks, of which 64 are reserved for Seen evaluation, and 200 Unseen test tasks. We group its tasks into finding, growth/lifespan, genetics, mixing, properties, and mechanics/energy.

\noindent\textbf{Baselines.}
We use Qwen3-1.7B and Qwen3-4B as students and Qwen3-32B as the teacher~\citep{yang2025qwen3}. All RLVR training uses GRPO. OPW uses cross-entropy on teacher top-1 targets at response-token positions in student-generated trajectories. We compare OPW with \textbf{(1)} off-policy SFT, \textbf{(2)} SFT with rejection-sampling (SFT-RS), and \textbf{(3)} teacher-free on-policy GRPO warmup as baselines. Each warmup method runs for 40 steps and is followed by 60 GRPO steps without further teacher supervision. We compare OPW + GRPO with direct OPD and direct GRPO over 100 steps.

\noindent\textbf{Evaluation metrics.}
Task performance is mean task success (pass@1, \%) in ALFWorld and normalized task scores (0--100) in ScienceWorld. We measure success coverage by pass@$k$ (at least one success in $k$ rollouts)~\citep{chen2021codex}, and repeated task success by pass\textasciicircum{}k (success in every rollout)~\citep{yao2024taubench,barres2025tau2bench}. ScienceWorld scores include partial progress, while success requires full task completion. Final results use eight trajectories per task, and the warmup success profile uses 64. Training and evaluation details are provided in Appendix~\ref{app:settings}.

\begin{table*}[t]
\caption{Final task performance after 40 warmup and 60 GRPO steps (Qwen3-4B).}
\label{tab:warmup-strategies-mean}
\centering\fontsize{7.5}{9}\selectfont
\renewcommand{\arraystretch}{1.12}
\setlength{\tabcolsep}{1pt}
\arrayrulecolor{black}
\begingroup
\begin{tabular}{@{\hspace{3pt}}>{\raggedright\arraybackslash}p{31pt}>{\raggedright\arraybackslash}p{78pt}*{7}{>{\centering\arraybackslash}p{17.8pt}}|*{7}{>{\centering\arraybackslash}p{17.8pt}}@{\hspace{3pt}}}
\toprule
\multirow{2}{*}{\textbf{Split}} & \multirow{2}{78pt}{\diagbox[width=78pt,height=20pt]{\textbf{Warmup}}{\textbf{Task type}}} & \multicolumn{7}{c|}{\textbf{ALFWorld}} & \multicolumn{7}{c}{\textbf{ScienceWorld}} \\
\cmidrule(lr){3-9}\cmidrule(l){10-16}
& & Pick & Clean & Heat & Cool & Exam & Pick2 & All & Find & Life & Gene. & Mix & Prop. & Phys. & All \\
\midrule
\multirow{4}{*}{\textbf{Seen}} & {\fontsize{7.5}{9}\selectfont Off-policy: SFT} & 97.5 & 82.4 & 75.8 & 74.5 & 66.3 & 90.6 & 83.9 & 69.9 & 49.3 & \subbst{67.8} & \bst{42.8} & 58.2 & 63.0 & 60.5 \\
& {\fontsize{7.5}{9}\selectfont Off-policy: SFT-RS} & \subbst{98.9} & \subbst{90.7} & 78.1 & \subbst{78.5} & \subbst{88.5} & \subbst{91.1} & \subbst{89.0} & 79.7 & 50.2 & 58.5 & 34.6 & 67.1 & 62.0 & 64.6 \\
& {\fontsize{7.5}{9}\selectfont On-policy: GRPO} & \bst{99.3} & 74.1 & \subbst{78.9} & 69.5 & 70.2 & \subbst{91.1} & 82.7 & \subbst{81.5} & \subbst{51.2} & 61.1 & \subbst{35.3} & \subbst{67.2} & \subbst{73.3} & \subbst{68.8} \\
\rowcolor{OPDLightBlue}
\cellcolor{white} & {\fontsize{7.5}{9}\selectfont \textbf{OPW (Ours)}} & \subbst{98.9} & \bst{98.1} & \bst{86.7} & \bst{93.5} & \bst{97.1} & \bst{97.9} & \bst{96.1} & \bst{97.0} & \bst{53.6} & \bst{78.7} & 30.4 & \bst{98.0} & \bst{82.4} & \bst{87.3} \\
\midrule
\multirow{4}{*}{\textbf{Unseen}} & {\fontsize{7.5}{9}\selectfont Off-policy: SFT} & \subbst{88.0} & 74.6 & \subbst{82.1} & 81.5 & 72.2 & 84.6 & 80.3 & \subbst{71.2} & 53.0 & 45.3 & 14.4 & 68.6 & \subbst{61.6} & 61.9 \\
& {\fontsize{7.5}{9}\selectfont Off-policy: SFT-RS} & \subbst{88.0} & \subbst{91.1} & 81.5 & \subbst{92.9} & \subbst{85.4} & \subbst{86.0} & \subbst{87.8} & 70.2 & 54.3 & 43.0 & 12.4 & \subbst{75.4} & 59.9 & \subbst{63.6} \\
& {\fontsize{7.5}{9}\selectfont On-policy: GRPO} & 85.4 & 78.2 & 71.2 & 77.4 & 66.7 & 77.2 & 76.5 & 68.4 & \subbst{54.5} & \subbst{50.0} & \bst{15.8} & 61.5 & 60.0 & 59.6 \\
\rowcolor{OPDLightBlue}
\cellcolor{white} & {\fontsize{7.5}{9}\selectfont \textbf{OPW (Ours)}} & \bst{93.8} & \bst{100.0} & \bst{92.4} & \bst{99.4} & \bst{97.2} & \bst{97.1} & \bst{96.7} & \bst{71.8} & \bst{62.5} & \bst{52.4} & \subbst{15.4} & \bst{75.8} & \bst{74.9} & \bst{68.2} \\
\bottomrule
\end{tabular}
\endgroup
\end{table*}

\begin{table*}[t]
\caption{Final repeated task success (pass\textasciicircum{}8, \%) for three training paths at 100 steps (Qwen3-4B).}
\label{tab:continuous-objectives-reliability}
\centering\fontsize{7.5}{9}\selectfont
\renewcommand{\arraystretch}{1.12}
\setlength{\tabcolsep}{1pt}
\arrayrulecolor{black}
\begingroup
\begin{tabular}{@{\hspace{3pt}}>{\raggedright\arraybackslash}p{31pt}>{\raggedright\arraybackslash}p{78pt}*{7}{>{\centering\arraybackslash}p{17.8pt}}|*{7}{>{\centering\arraybackslash}p{17.8pt}}@{\hspace{3pt}}}
\toprule
\multirow{2}{*}{\textbf{Split}} & \multirow{2}{78pt}{\diagbox[width=78pt,height=20pt]{\textbf{Training}}{\textbf{Task type}}} & \multicolumn{7}{c|}{\textbf{ALFWorld}} & \multicolumn{7}{c}{\textbf{ScienceWorld}} \\
\cmidrule(lr){3-9}\cmidrule(l){10-16}
& & Pick & Clean & Heat & Cool & Exam & Pick2 & All & Find & Life & Gene. & Mix & Prop. & Phys. & All \\
\midrule
\multirow{3}{*}{\textbf{Seen}} & {\fontsize{7.5}{9}\selectfont Direct OPD} & 80.0 & 25.9 & 0.0 & 0.0 & 23.1 & 16.7 & 30.0 & 0.0 & 0.0 & 0.0 & 0.0 & 0.0 & \subbst{4.8} & \subbst{1.6} \\
& {\fontsize{7.5}{9}\selectfont Direct GRPO} & \bst{100.0} & \subbst{48.1} & \subbst{68.8} & \subbst{60.0} & \subbst{38.5} & \bst{87.5} & \subbst{71.4} & 0.0 & 0.0 & 0.0 & 0.0 & 0.0 & 0.0 & 0.0 \\
\rowcolor{OPDLightBlue}
\cellcolor{white} & {\fontsize{7.5}{9}\selectfont \textbf{OPW (Ours)} + GRPO} & \subbst{97.1} & \bst{92.6} & \bst{81.2} & \bst{80.0} & \bst{84.6} & \subbst{83.3} & \bst{87.9} & \bst{62.5} & \bst{25.0} & \bst{50.0} & 0.0 & \bst{63.0} & \bst{19.0} & \bst{43.8} \\
\midrule
\multirow{3}{*}{\textbf{Unseen}} & {\fontsize{7.5}{9}\selectfont Direct OPD} & \subbst{83.3} & \subbst{41.9} & 0.0 & 9.5 & 11.1 & 5.9 & 28.4 & 1.9 & \subbst{2.9} & \subbst{12.0} & 0.0 & 0.0 & 0.0 & 2.5 \\
& {\fontsize{7.5}{9}\selectfont Direct GRPO} & 79.2 & 38.7 & \subbst{56.5} & \subbst{61.9} & \subbst{72.2} & \subbst{76.5} & \subbst{61.9} & \bst{9.6} & \subbst{2.9} & 8.0 & 0.0 & \subbst{3.1} & \bst{21.1} & \subbst{7.0} \\
\rowcolor{OPDLightBlue}
\cellcolor{white} & {\fontsize{7.5}{9}\selectfont \textbf{OPW (Ours)} + GRPO} & \bst{87.5} & \bst{100.0} & \bst{82.6} & \bst{95.2} & \bst{77.8} & \bst{82.4} & \bst{88.8} & \subbst{3.8} & \bst{14.3} & \bst{16.0} & 0.0 & \bst{4.6} & \subbst{15.8} & \bst{8.5} \\
\bottomrule
\end{tabular}
\endgroup
\end{table*}

\begin{table*}[t]
\caption{Success coverage and repeated task success after warmup (pass@$k$/pass\textasciicircum{}k, \%).}
\label{tab:warmup-profile}
\par\smallskip
\centering\fontsize{8.5}{10.5}\selectfont
\setlength{\tabcolsep}{1.5pt}
\renewcommand{\arraystretch}{1.12}
\newcommand{\warmuppair}[2]{#1\hspace{0.4pt}/\hspace{0.4pt}#2}
\begingroup
\begin{tabular}{@{\hspace{6pt}}p{88pt}*{8}{>{\centering\arraybackslash}p{\dimexpr(\textwidth-124pt)/8\relax}}@{\hspace{6pt}}}
\toprule
\multirow{2}{88pt}{\diagbox[width=88pt,height=22pt]{\textbf{Policy}}{\textbf{Split / $k$}}} & \multicolumn{4}{c}{\textbf{Seen}} & \multicolumn{4}{c}{\textbf{Unseen}} \\
\cmidrule(lr){2-5}\cmidrule(lr){6-9}
& $k=8$ & $k=16$ & $k=32$ & $k=64$ & $k=8$ & $k=16$ & $k=32$ & $k=64$ \\
\midrule
\rowcolor{TableHeader}
\multicolumn{9}{@{\hspace{6pt}}c@{\hspace{6pt}}}{\textbf{ALFWorld}} \\
\addlinespace[2pt]
Teacher (32B) & \warmuppair{\subbst{70.9}}{27.5} & \warmuppair{\subbst{76.4}}{21.5} & \warmuppair{\subbst{81.0}}{15.2} & \warmuppair{\subbst{85.7}}{10.0} & \warmuppair{\bst{88.3}}{\subbst{24.2}} & \warmuppair{\bst{93.8}}{\subbst{17.3}} & \warmuppair{\bst{97.1}}{\subbst{11.5}} & \warmuppair{\bst{99.3}}{\subbst{6.7}} \\
Base (4B) & \warmuppair{51.5}{15.8} & \warmuppair{58.4}{13.2} & \warmuppair{64.6}{10.9} & \warmuppair{69.3}{9.3} & \warmuppair{55.2}{10.9} & \warmuppair{63.1}{8.0} & \warmuppair{70.7}{5.9} & \warmuppair{78.4}{4.5} \\
\cmidrule(lr){1-9}
{\fontsize{7.5}{9}\selectfont Off-policy: SFT} & \warmuppair{52.9}{12.6} & \warmuppair{60.0}{9.3} & \warmuppair{66.1}{6.7} & \warmuppair{71.4}{4.3} & \warmuppair{58.7}{7.7} & \warmuppair{67.8}{5.2} & \warmuppair{75.8}{3.8} & \warmuppair{82.8}{3.0} \\
{\fontsize{7.5}{9}\selectfont Off-policy: SFT-RS} & \warmuppair{53.2}{12.6} & \warmuppair{60.7}{9.1} & \warmuppair{67.5}{6.4} & \warmuppair{73.6}{4.3} & \warmuppair{59.8}{7.0} & \warmuppair{69.0}{3.7} & \warmuppair{76.8}{1.9} & \warmuppair{84.3}{1.5} \\
{\fontsize{7.5}{9}\selectfont On-policy: GRPO} & \warmuppair{\bst{78.6}}{\bst{28.4}} & \warmuppair{\bst{83.6}}{\subbst{23.0}} & \warmuppair{\bst{87.9}}{\subbst{18.6}} & \warmuppair{\bst{91.4}}{\subbst{15.0}} & \warmuppair{80.1}{19.8} & \warmuppair{84.8}{13.9} & \warmuppair{88.1}{9.0} & \warmuppair{90.3}{5.2} \\
\rowcolor{OPDLightBlue}
{\fontsize{7.5}{9}\selectfont \textbf{OPW (Ours)}} & \warmuppair{69.3}{\subbst{28.1}} & \warmuppair{75.2}{\bst{24.6}} & \warmuppair{79.5}{\bst{22.4}} & \warmuppair{82.9}{\bst{21.4}} & \warmuppair{\subbst{88.1}}{\bst{27.5}} & \warmuppair{\subbst{93.2}}{\bst{22.9}} & \warmuppair{\subbst{95.0}}{\bst{19.7}} & \warmuppair{\subbst{95.5}}{\bst{17.9}} \\
\midrule
\rowcolor{TableHeader}
\multicolumn{9}{@{\hspace{6pt}}c@{\hspace{6pt}}}{\textbf{ScienceWorld}} \\
\addlinespace[2pt]
Teacher (32B) & \warmuppair{\bst{75.8}}{0.6} & \warmuppair{\bst{85.0}}{0.0} & \warmuppair{\bst{91.9}}{0.0} & \warmuppair{\bst{95.3}}{0.0} & \warmuppair{\bst{76.3}}{3.1} & \warmuppair{\bst{84.5}}{1.1} & \warmuppair{\bst{88.6}}{0.2} & \warmuppair{\bst{90.5}}{0.0} \\
Base (4B) & \warmuppair{23.9}{0.0} & \warmuppair{28.9}{0.0} & \warmuppair{32.9}{0.0} & \warmuppair{35.9}{0.0} & \warmuppair{35.9}{2.9} & \warmuppair{44.6}{1.7} & \warmuppair{53.5}{\subbst{0.8}} & \warmuppair{62.0}{0.0} \\
\cmidrule(lr){1-9}
{\fontsize{7.5}{9}\selectfont Off-policy: SFT} & \warmuppair{25.6}{0.0} & \warmuppair{31.6}{0.0} & \warmuppair{37.0}{0.0} & \warmuppair{42.2}{0.0} & \warmuppair{37.1}{2.1} & \warmuppair{45.8}{0.7} & \warmuppair{54.4}{0.1} & \warmuppair{63.0}{0.0} \\
{\fontsize{7.5}{9}\selectfont Off-policy: SFT-RS} & \warmuppair{22.0}{0.0} & \warmuppair{27.2}{0.0} & \warmuppair{31.9}{0.0} & \warmuppair{35.9}{0.0} & \warmuppair{38.0}{2.1} & \warmuppair{46.6}{0.7} & \warmuppair{55.2}{0.1} & \warmuppair{63.5}{0.0} \\
{\fontsize{7.5}{9}\selectfont On-policy: GRPO} & \warmuppair{51.5}{\subbst{1.4}} & \warmuppair{56.3}{\subbst{0.1}} & \warmuppair{59.5}{0.0} & \warmuppair{62.5}{0.0} & \warmuppair{62.9}{\bst{5.4}} & \warmuppair{71.1}{\bst{2.3}} & \warmuppair{77.4}{\subbst{0.8}} & \warmuppair{82.5}{\bst{0.5}} \\
\rowcolor{OPDLightBlue}
{\fontsize{7.5}{9}\selectfont \textbf{OPW (Ours)}} & \warmuppair{\subbst{73.7}}{\bst{7.0}} & \warmuppair{\subbst{82.3}}{\bst{3.5}} & \warmuppair{\subbst{87.6}}{\bst{1.7}} & \warmuppair{\subbst{90.6}}{0.0} & \warmuppair{\subbst{66.8}}{\subbst{4.5}} & \warmuppair{\subbst{75.9}}{\subbst{2.0}} & \warmuppair{\subbst{82.4}}{\bst{0.9}} & \warmuppair{\subbst{85.5}}{\bst{0.5}} \\
\bottomrule
\addlinespace[2pt]
\multicolumn{9}{@{\hspace{6pt}}l@{}}{\fontsize{8}{9.5}\selectfont Each cell: pass@$k$/pass\textasciicircum{}k (\%). \bst{Best} and \subbst{second-best} nonzero values per metric and column, including Teacher.} \\
\end{tabular}
\endgroup
\end{table*}

\begin{figure*}[t]
\centering
\includegraphics[width=\textwidth]{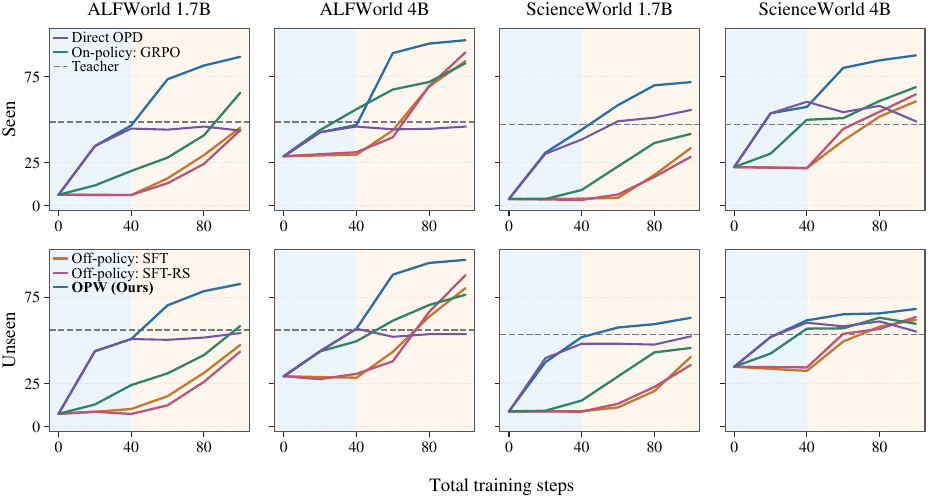}
\caption{Task performance during warmup and GRPO.}
\label{fig:warmup-cross-environment}
\end{figure*}

\subsection{Results}

\noindent\textbf{Warmup VS. Direct Training.}
OPW followed by GRPO achieves the best final task performance among the three training paths in both environments (Table~\ref{tab:continuous-objectives-final}). On ALFWorld Unseen, it reaches 96.7\% success, compared with 86.6\% for direct GRPO and 53.5\% for direct OPD. The same ordering holds in ScienceWorld, where Unseen task scores reach 68.2, 61.2, and 55.2, respectively. These results support on-policy distillation as an effective warmup for agentic RLVR.

\noindent\textbf{Comparison of Warmup Strategies.}
OPW yields the highest final task performance under the same GRPO procedure after warmup (Table~\ref{tab:warmup-strategies-mean}). On ALFWorld Unseen, restricting SFT to successful teacher trajectories raises final success from 80.3\% to 87.8\%, while OPW reaches 96.7\%, compared with 76.5\% after GRPO warmup. In ScienceWorld, OPW exceeds the strongest alternative by 18.5 points on Seen tasks and 4.6 points on Unseen tasks. Gains over both SFT variants and teacher-free warmup support teacher supervision on student-generated trajectories during warmup.

\noindent\textbf{Task-Specific Performance.}
In ALFWorld, Clean, Heat, and Cool require changing an object's state before placement, adding intermediate requirements to the retrieval and placement needed for Pick~\citep{shridhar2021alfworld}. OPW improves these tasks in both comparisons (Tables~\ref{tab:continuous-objectives-final} and~\ref{tab:warmup-strategies-mean}). In the warmup comparison, Seen Pick success already exceeds 97\% for all methods, while OPW improves Clean and Cool success over GRPO warmup by 24.0 percentage points each. In ScienceWorld, Mixing requires selecting and combining materials to obtain a target product~\citep{wang2022scienceworld}. OPW does not lead this category despite its higher scores on Properties and Mechanics/Energy, showing that its gains do not extend uniformly across scientific tasks (Table~\ref{tab:warmup-strategies-mean}).

\noindent\textbf{Repeated Task Success.}
OPW improves repeated task success in both environments (Table~\ref{tab:continuous-objectives-reliability}). On ALFWorld Unseen, final pass\textasciicircum{}8 reaches 88.8\%, compared with 61.9\% for direct GRPO and 28.4\% for direct OPD. In ScienceWorld, OPW followed by GRPO reaches final pass\textasciicircum{}8 of 43.8\% on Seen tasks and 8.5\% on Unseen tasks, compared with 0.0\% and 7.0\% for direct GRPO. Thus, more tasks are completed successfully in every rollout, although repeated task success remains difficult on ScienceWorld Unseen. The comparison after warmup further distinguishes success coverage from repeated task success (Table~\ref{tab:warmup-profile}). On ALFWorld Seen, OPW has lower pass@64 than GRPO warmup (82.9\% versus 91.4\%) but higher pass\textasciicircum{}64 (21.4\% versus 15.0\%). Its stronger final task performance therefore does not require the broadest initial success coverage. We next examine how warmup affects subsequent learning and the behavior generated during teacher-free GRPO.

\begin{table}[t]
\centering
\begin{minipage}[c]{0.62\textwidth}
\centering
\caption{Invalid actions after warmup (Qwen3-4B). Values are task-mean shares of all actions (\%) across Seen and Unseen.}
\label{tab:warmup-action-composition}
\fontsize{8}{9.5}\selectfont
\setlength{\tabcolsep}{1.5pt}
\renewcommand{\arraystretch}{1.08}
\arrayrulecolor{black}
\begingroup
\begin{tabular}{@{\hspace{4pt}}>{\raggedright\arraybackslash}p{52pt}*{3}{>{\centering\arraybackslash}p{\dimexpr(\linewidth-78pt-\arrayrulewidth)/6\relax}}|*{3}{>{\centering\arraybackslash}p{\dimexpr(\linewidth-78pt-\arrayrulewidth)/6\relax}}@{\hspace{4pt}}}
\toprule
\multirow{2}{52pt}{\diagbox[width=52pt,height=20pt]{\textbf{Warmup}}{\textbf{Type}}} & \multicolumn{3}{c|}{\textbf{ALFWorld}} & \multicolumn{3}{c}{\textbf{ScienceWorld}} \\
\cmidrule(lr){2-4}\cmidrule(lr){5-7}
& Nonrep. & Repeat & Consec. & Nonrep. & Repeat & Consec. \\
\midrule
SFT & 3.05 & 10.10 & 5.90 & 16.23 & 24.91 & 13.68 \\
SFT-RS & 3.12 & 10.24 & 5.90 & 16.11 & 23.79 & 12.88 \\
GRPO & 3.52 & 6.78 & 5.25 & 29.05 & 23.15 & 8.01 \\
\rowcolor{OPDLightBlue}
\textbf{OPW (Ours)} & 3.14 & 3.19 & 0.53 & 30.62 & 19.17 & 4.02 \\
\bottomrule
\end{tabular}
\endgroup
\end{minipage}\hfill
\begin{minipage}[c]{0.35\textwidth}
\centering
\caption{Task performance for fixed and resampled trajectories (ALFWorld 4B, Seen/Unseen).}
\label{tab:student-resampling-summary}
\fontsize{8}{9.5}\selectfont
\setlength{\tabcolsep}{1pt}
\renewcommand{\arraystretch}{1.08}
\begingroup
\begin{tabular}{@{\hspace{2pt}}>{\raggedright\arraybackslash}p{51pt}>{\centering\arraybackslash}p{34pt}>{\columncolor{OPDLightBlue}\centering\arraybackslash}p{\dimexpr\linewidth-93pt\relax}@{\hspace{2pt}}}
\toprule
\textbf{Metric} & \textbf{Fixed} & \textbf{OPW (Ours)} \\
\midrule
Early gain (pp) & 38.6\,/\,29.2 & 41.8\,/\,31.4 \\
Average (\%) & 81.1\,/\,81.4 & 84.7\,/\,86.6 \\
Final (\%) & 97.4\,/\,95.1 & 96.1\,/\,96.7 \\
\bottomrule
\end{tabular}
\endgroup
\end{minipage}
\end{table}

\begin{figure*}[t]
\centering
\includegraphics[width=\textwidth]{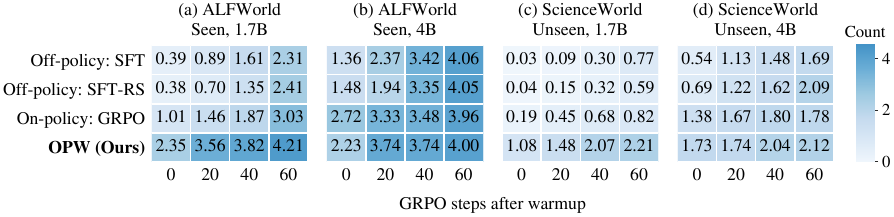}
\caption{Distinct successful action sequences among eight rollouts per task during GRPO.}
\label{fig:successful-sequence-dynamics}
\end{figure*}

\section{More Analysis}
\label{sec:analysis}

\noindent\textbf{Learning Efficiency and Training Cost.} OPW improves subsequent learning even from lower warmup success. On ALFWorld 4B Seen, it starts below GRPO warmup (46.8\% versus 56.0\%) but gains 41.8 versus 11.4 percentage points in the first 20 GRPO steps after warmup. Its average success over all 60 steps is 84.7\% versus 69.5\% (Fig.~\ref{fig:warmup-cross-environment}). OPW has the highest average task performance across environments, model sizes, and splits, even on ScienceWorld 1.7B Unseen, where GRPO warmup yields a larger early gain. Warmup scores alone do not explain subsequent learning.

OPW exceeds 80\% Seen success on ALFWorld 4B at 19.6 recorded training GPU-hours including teacher supervision (Appendix~\ref{app:cross-environment}), versus at least 29.6 for other methods. Only OPW exceeds 90\% within 100 steps. Over 100 steps, OPW costs less than GRPO warmup at 4B but more at 1.7B.

\noindent\textbf{Repeated Invalid Actions.} OPW's main reduction in ALFWorld action errors concerns repeated invalid commands. We classify validity by membership in the current valid-action list and repetition by earlier use in the trajectory. Unparseable responses are invalid. Repeated invalid commands account for 3.19\% of actions after OPW, versus 10.10\% after SFT and 6.78\% after GRPO warmup, while non-repeated invalid actions remain near 3\% (Table~\ref{tab:warmup-action-composition}). In ScienceWorld, OPW reduces repeated errors but produces more non-repeated invalid actions than SFT. This list-based measure can also mark accepted navigation aliases as invalid. Similar fractions of SFT and OPW trajectories contain an invalid action on ALFWorld Seen (50.9\% and 50.0\%), despite different levels of repetition. Appendix~\ref{app:trajectory-cases} illustrates these errors and valid repetition in two complete action sequences.

To test the contribution of supervision at these errors, we remove teacher supervision at repeated invalid actions. Early GRPO gain on ALFWorld 4B Seen falls to 34.6 percentage points, compared with 42.0 under random removal of a similar number of action tokens (Table~\ref{tab:error-supervision-learning}). Warmup action validity remains similar. Supervision at repeated errors contributes to subsequent learning.

\begin{figure*}[t]
\centering
\includegraphics[width=\textwidth]{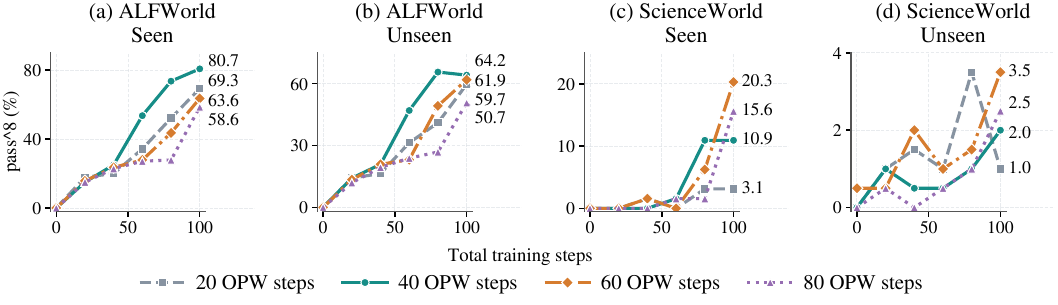}
\caption{Repeated task success (pass\textasciicircum{}8) across OPW durations (Qwen3-1.7B).}
\label{fig:warmup-reliability}
\end{figure*}

\noindent\textbf{Teacher Supervision on Student-Generated Trajectories.} Teacher supervision provides a learning signal on unsuccessful student trajectories. In all-failure groups on ALFWorld 4B, one optimizer update from Base raises mean teacher-target log probability at teacher--student disagreements by 0.183 with OPD, versus 0.128 with SFT and 0.004 with GRPO (Appendix~\ref{app:supervision}). Compared with reusing initial student trajectories, resampling improves average success during GRPO by 3.6 percentage points on Seen and 5.2 on Unseen (Table~\ref{tab:student-resampling-summary}). Fixed trajectories still yield higher final Seen success, so resampling mainly benefits early learning and average success during GRPO.

\noindent\textbf{Task Value and Retention of Teacher Corrections.} Teacher substitution raises continuation success from 26.48\% to 29.70\% at 311 selected action-token disagreements (Appendix~\ref{app:expanded-action-value}). Base continues both branches. Over 40 teacher-free GRPO steps, teacher-target probability rises from 39.3\% to 47.2\% at these positions but changes little on a broader disagreement set, from 43.8\% to 43.2\% (Table~\ref{tab:warmup-choice-persistence}). The broader set shows retention, with further gains at selected action positions.

A complementary evaluation (Appendix~\ref{app:action-value-analysis}) samples complete actions from each policy at common histories, with Base taking all subsequent actions. At 447 active positions from 63 tasks, continuation success increases from 45.00\% to 46.75\% over 40 GRPO steps after OPW, compared with 45.93\% to 46.47\% after GRPO warmup. OPW again improves more from a lower starting score. With Base fixed, these gains measure action selection at the evaluated histories.

\noindent\textbf{Diversity of Successful Action Sequences.} OPW samples more distinct successful action sequences early in GRPO across both environments and model sizes (Fig.~\ref{fig:successful-sequence-dynamics}). After sequence normalization (Appendix~\ref{app:behavior}), its mean count among eight rollouts per task rises from 2.23 to 4.00 over 60 GRPO steps on ALFWorld 4B Seen. Failure-only sequences fall from 2.95 to 0.31, accounting for the lower total sequence count. Among four selected successful trajectories on common tasks with enough successes under every method and checkpoint, OPW's expected distinct-sequence count falls from 2.54 to 2.14 on ALFWorld 4B but rises from 1.83 to 2.32 on ScienceWorld 4B. More successful sequences in eight rollouts need not mean more varied successful trajectories.

\noindent\textbf{Choosing the Warmup Duration.} Teacher supervision can improve repeated task success after most actions become valid. On ALFWorld 1.7B, extending OPW from 20 to 40 steps raises validity only from 89.4\% to 91.8\%, while pass\textasciicircum{}8 more than doubles from 10.8\% to 22.0\%. Fig.~\ref{fig:warmup-reliability} compares 20, 40, 60, and 80 OPW steps followed by GRPO to 100 total steps, with varying seeds. The ALFWorld 20/40 comparison branches from the same warmup run. On Seen tasks, longer warmup raises final pass\textasciicircum{}8 from 69.3\% to 80.7\% and moves the first evaluation reaching 80\% success from total step 75 to 60. In a separate ScienceWorld comparison from one 40-step OPW checkpoint, switching immediately gives higher Unseen pass\textasciicircum{}8 than adding 20 OPW steps (4.5\% versus 3.5\%), despite a lower mean task score. Choose warmup duration by subsequent validation on the required task outcome. Additional duration comparisons and warmup diagnostics are provided in Appendix~\ref{app:process-switching}.

\section{Conclusion}

We study On-Policy Warmup (OPW), which uses teacher supervision on student-generated trajectories before RLVR. Our main comparisons in ALFWorld and ScienceWorld show the On-Policy Acceleration Phenomenon, with OPW reaching high performance earlier and achieving higher average and final task performance. Our theory connects on-policy reverse-KL distillation to trajectory-level distribution matching, initial verifier success, and reward-discovery complexity under its stated conditions. Empirical analyses show that teacher-target probability advantages persist during teacher-free GRPO, alongside improved action selection and repeated task success. Together, these findings support OPW as an effective warmup for agentic RLVR.

\section*{AI Use Statement}
AI assistants were used to refine the language and improve the clarity of the manuscript. All scientific claims, theoretical analyses, experimental results, and references were reviewed and verified by the authors, who take full responsibility for the final manuscript.

\section*{Reproducibility Statement}
Section~\ref{sec:method} defines OPW and connects population reverse-KL distillation to initial verifier success and reward discovery, with assumptions and complete proofs in Appendix~\ref{app:proofs}. Section~\ref{sec:comparison} specifies the models, benchmarks, baselines, and evaluation metrics. Appendix~\ref{app:settings} details task partitions, agent interaction, the teacher top-1 cross-entropy implementation, optimization settings, rollout sampling, and metric estimators. The analyses in Section~\ref{sec:analysis} are supported by protocols for teacher-target probability measurements, trajectory resampling, and supervision removal in Appendix~\ref{app:supervision}. Appendix~\ref{app:action-value-analysis} describes token and complete-action interventions at common histories, including position selection, continuation policies, aggregation, and measurement of teacher-target preferences during teacher-free GRPO. Appendix~\ref{app:behavior} defines action-error and successful action-sequence measurements. Appendices~\ref{app:cross-environment} and~\ref{app:process-switching} document additional training-length and seed comparisons, training-cost accounting including teacher supervision, and warmup-duration allocations. Appendix~\ref{app:trajectory-cases} provides complete action traces with selected environment feedback.

\bibliographystyle{iclr2027_conference}
\bibliography{references}

\clearpage
\appendix
\addtocontents{toc}{\protect\setcounter{tocdepth}{2}}
\begingroup
\renewcommand{\contentsname}{Appendix Contents}
\hypersetup{linktoc=all}
\tableofcontents
\endgroup
\section{Proofs and Theoretical Details}
\label{app:proofs}

\subsection{Trajectory KL and Success Transfer}

Fix a task $x$ and suppress its index. The student $\pi$ and teacher $\pi_T$ induce distributions $P_\pi$ and $P_T$ over horizon $H$, with the same initial distribution $\rho$ and transition kernels $K_t(\cdot\mid s_t,a_t)$. States contain full interaction histories, and $d_t^\pi$ is the student's state distribution. Padding uses an absorbing state and a shared deterministic null action with zero KL. We use the cumulative objective in Eq.~\ref{eq:opd-loss} and natural logarithms, with $0\log(0/q)=0$ and $p\log(p/0)=+\infty$ for $p>0$.

\begin{proof}[Proof of Proposition~\ref{prop:trajectory-kl}]
Write a trajectory as
$\tau=(s_1,a_1,\ldots,s_H,a_H,s_{H+1})$.
Using probability-mass notation, its probability under the student
factorizes as
\[
P_\pi(\tau)
=
\rho(s_1)
\prod_{t=1}^{H}
\pi(a_t \mid s_t)
K_t(s_{t+1} \mid s_t,a_t),
\]
and the teacher trajectory probability $P_T(\tau)$ has the same factorization with $\pi$ replaced by $\pi_T$.

First suppose that $P_\pi \ll P_T$.
The initial-state and environment-transition factors are identical
under the two policies and therefore cancel in the likelihood ratio.
Consequently, $P_\pi$-almost surely,
\[
\log \frac{P_\pi(\tau)}{P_T(\tau)}
=
\sum_{t=1}^{H}
\log
\frac{\pi(a_t \mid s_t)}
     {\pi_T(a_t \mid s_t)}.
\]
Taking expectation under $P_\pi$ and conditioning on $s_t$ yields
\begin{align*}
D_{\mathrm{KL}}(P_\pi \| P_T)
&=
\mathbb{E}_{\tau \sim P_\pi}
\left[
\sum_{t=1}^{H}
\log
\frac{\pi(a_t \mid s_t)}
     {\pi_T(a_t \mid s_t)}
\right]
\\
&=
\sum_{t=1}^{H}
\mathbb{E}_{s_t \sim d_t^\pi}
\left[
\mathbb{E}_{a_t \sim \pi(\cdot \mid s_t)}
\left[
\log
\frac{\pi(a_t \mid s_t)}
     {\pi_T(a_t \mid s_t)}
\right]
\right]
\\
&=
\sum_{t=1}^{H}
\mathbb{E}_{s_t \sim d_t^\pi}
\left[
D_{\mathrm{KL}}
\left(
\pi(\cdot \mid s_t)
\,\middle\|\,
\pi_T(\cdot \mid s_t)
\right)
\right]
\\
&=
\mathcal{L}_{\mathrm{OPD}}(\pi).
\end{align*}

If $P_\pi \not\ll P_T$, the shared initial distribution and
transition kernels imply an action-support mismatch on a set of
student-visited states with positive probability.
Both the expected conditional and trajectory KL divergences are infinite.
The trajectory identity therefore holds under these conventions.
\end{proof}

\begin{proof}[Proof of Theorem~\ref{thm:success-transfer}]
Let $V(\tau)=R_x(\tau)$ be the shared binary verifier. Its success probabilities are $p_\pi$ and $p_T$ as defined in the main text. For Bernoulli distributions, $\operatorname{KL}(p\|q)=p\log(p/q)+(1-p)\log((1-p)/(1-q))$, using the same endpoint conventions.

Applying $V$ maps $P_\pi$ and $P_T$ to $\operatorname{Bern}(p_\pi)$ and $\operatorname{Bern}(p_T)$.
The data-processing inequality
\citep[Theorem~9]{vanerven2014renyi} therefore gives
\begin{align*}
\operatorname{KL}(p_\pi\|p_T)
&=
D_{\mathrm{KL}}
\left(
\operatorname{Bern}(p_\pi)
\,\middle\|\,
\operatorname{Bern}(p_T)
\right)
\\
&\leq
D_{\mathrm{KL}}(P_\pi\|P_T)
\\
&=
\mathcal{L}_{\mathrm{OPD}}(\pi)
\\
&\leq \varepsilon,
\end{align*}
where the equality in the third line follows from
Proposition~\ref{prop:trajectory-kl}.

For Bernoulli distributions, Pinsker's inequality
\citep[Theorem~31]{vanerven2014renyi} yields
\[
2|p_\pi-p_T|^2
\leq
\operatorname{KL}(p_\pi\|p_T)
\leq \varepsilon.
\]
Taking square roots gives the gap bound in
Eq.~\ref{eq:success-bound}.
Rearranging with $p_\pi\geq 0$ gives the lower bound.
\end{proof}

\noindent\textbf{Offline state coverage.}
Evaluating the same local reverse KL under offline distributions $\mu_t$ requires $d_t^\pi\ll\mu_t$ and a finite density-ratio bound $\frac{d d_t^\pi}{d\mu_t}\leq C$ to recover the bound. Nonnegativity of local KL then gives $\mathcal{L}_{\mathrm{OPD}}(\pi)\leq C\mathcal{L}_{\mu}(\pi)$. No finite $C$ exists if absolute continuity fails. 

\noindent\textbf{Normalization and scope.}
For a fixed horizon, $\overline{\mathcal{L}}_{\mathrm{OPD}}=\mathcal{L}_{\mathrm{OPD}}/H$ gives $D_{\mathrm{KL}}(P_\pi\|P_T)=H\overline{\mathcal{L}}_{\mathrm{OPD}}$. Thus, $\overline{\mathcal{L}}_{\mathrm{OPD}}\leq\delta$ implies $p_\pi\geq[p_T-\sqrt{H\delta/2}]_+$. Normalization by realized trajectory length need not preserve this identity. The guarantee uses the population objective for the specified task, environment, verifier, and sampling policies. Certifying this premise from empirical estimates or under new decoding rules or evaluation distributions requires further error bounds.

\subsection{Reward Discovery and Reward-Informative Groups}

\begin{proof}[Proof of Corollary~\ref{cor:reward-discovery}]
For independent rollouts from the fixed policy $\pi_{\mathrm{w}}$, the first-success time $T_{\mathrm{succ}}$ is geometric with parameter $p_{\mathrm{w}}$. When $p_{\mathrm{w}}\geq\ell>0$,
\[
\Pr(T_{\mathrm{succ}}>N)=(1-p_{\mathrm{w}})^N
\leq(1-\ell)^N\leq e^{-N\ell},
\qquad
\mathbb{E}[T_{\mathrm{succ}}]=\frac{1}{p_{\mathrm{w}}}\leq\frac{1}{\ell}.
\]
Taking complements gives the success-probability bound. For $\delta\in(0,1)$, choosing $N\geq\lceil\log(1/\delta)/\ell\rceil$ makes the failure probability at most $\delta$.
\end{proof}

\noindent\textbf{Binary reward groups.}
Under binary rewards, a group of $G\geq2$ independent rollouts has all-success probability $p^G$, all-failure probability $(1-p)^G$, and mixed-group probability
\begin{equation}
 m_G(p)=1-p^G-(1-p)^G,
 \qquad
 m'_G(p)=G\big[(1-p)^{G-1}-p^{G-1}\big].
\label{eq:reward-contrast}
\end{equation}
This probability increases below $p=0.5$ and decreases above it. The lower bound $p_{\mathrm{w}}\geq\ell$ alone does not imply an increase in mixed groups, since it does not also ensure $p_{\mathrm{w}}\leq0.5$. All-success groups indicate reliable completion even though their direct GRPO advantages are zero. Empirical group fractions average these per-task quantities over tasks with different success probabilities. In ScienceWorld, partial-credit rewards can differ even without complete success.

\section{Training and Evaluation Protocol}
\label{app:settings}

\subsection{Data, Models, and Agent Interaction}

The students and teacher use released Qwen3 weights~\citep{yang2025qwen3} without additional agent-task training. Table~\ref{tab:training-settings} lists model sizes and shared settings. Chat templates disable thinking, while task prompts request \texttt{<reasoning>} and \texttt{<action>} fields. ALFWorld supplies the current observation and admissible actions. ScienceWorld supplies the task, observation, action templates, and objects. Both include the two most recent observation--action pairs when available.

\begingroup
\small
\setlength{\tabcolsep}{8pt}
\renewcommand{\arraystretch}{1.04}
\setlength{\LTcapwidth}{\linewidth}
\begin{longtable}{@{\hspace{4pt}}>{\raggedright\arraybackslash}p{142pt}>{\raggedright\arraybackslash}p{\dimexpr\linewidth-166pt\relax}@{\hspace{4pt}}}
\caption{Training and evaluation settings for the main warmup comparisons.}
\label{tab:training-settings} \\
\toprule
\textbf{Parameter} & \textbf{Value} \\
\midrule
\endfirsthead
\multicolumn{2}{c}{\tablename~\thetable\ (continued)} \\
\toprule
\textbf{Parameter} & \textbf{Value} \\
\midrule
\endhead
\bottomrule
\endfoot
\bottomrule
\endlastfoot
\rowcolor{TableHeader}
\multicolumn{2}{c}{\textbf{Models and training schedule}} \\*
Student & Qwen3-1.7B / Qwen3-4B \\*
Teacher & Qwen3-32B \\*
Environments & \textbf{ALFWorld} / \textbf{ScienceWorld} \\*
Main training length & 40 warmup + 60 GRPO steps \\*
Shorter warmup comparison & 20 warmup + 80 GRPO steps \\*
Teacher during GRPO & Not used \\
\addlinespace[3pt]
\rowcolor{TableHeader}
\multicolumn{2}{c}{\textbf{Objectives}} \\*
SFT / SFT-RS & Token cross-entropy on all / successful teacher trajectories \\*
OPD & Cross-entropy on teacher top-1 targets along student trajectories \\*
GRPO & Group-relative log-probability objective \\*
GRPO advantage normalization & Within-group mean and sample standard deviation, $\varepsilon=10^{-6}$ \\*
Probability-ratio clipping & Not used \\*
Additional KL / entropy coefficient & $0$ / $0$ \\
\addlinespace[3pt]
\rowcolor{TableHeader}
\multicolumn{2}{c}{\textbf{Optimization}} \\*
Optimizer & AdamW~\citep{loshchilov2019adamw}, $\beta=(0.9,0.999)$ \\*
Learning rate / schedule & $10^{-6}$ / constant \\*
Weight decay / gradient clipping & $0.01$ / norm $1.0$ \\*
Precision & bfloat16 \\*
Loss within a microbatch & Mean over response tokens \\
\addlinespace[3pt]
\rowcolor{TableHeader}
\multicolumn{2}{c}{\textbf{Batching and step counts}} \\*
OPD / GRPO batch & 16 tasks $\times$ 8 rollouts \\*
OPD / GRPO minibatch & 8 tasks $\times$ 8 rollouts \\*
Passes over each collected batch & 1 \\*
SFT / SFT-RS global batch & 128 teacher trajectories \\*
Optimizer steps per training step & SFT: 1, OPD / GRPO: 2 \\
\addlinespace[3pt]
\rowcolor{TableHeader}
\multicolumn{2}{c}{\textbf{Agent interaction}} \\*
Prompt / trajectory response limit & 4,096 / 24,576 tokens \\*
Response limit per turn & 512 tokens \\*
Maximum interaction length & 30 turns \\*
Recent interactions in the prompt & 2 observation--action pairs \\*
Chat-template thinking mode & Disabled \\
\addlinespace[3pt]
\rowcolor{TableHeader}
\multicolumn{2}{c}{\textbf{Sampling and evaluation}} \\*
Training temperature / top-$p$ & $1.0$ / $1.0$ \\*
Evaluation temperature / top-$p$ & $0.7$ / $0.95$ \\*
Learning-curve rollouts per task & 8 \\*
Rollouts per task for warmup success profile & 64 \\
\addlinespace[3pt]
\rowcolor{TableHeader}
\multicolumn{2}{c}{\textbf{Compute}} \\*
GPUs per training run & 8 \\*
Student / teacher placement & Shared eight-GPU pool for OPD \\*
Sequence parallel size & 4 \\
\end{longtable}
\endgroup

ALFWorld reserves 64 of its 3,553 training tasks for analysis, leaving disjoint warmup and GRPO pools of 1,744 and 1,745 tasks. ScienceWorld reserves 64 of 2,294 tasks and uses two training pools of 1,115. Both partitions use seed 20260722 and task-instance identities, allowing task types to overlap. ALFWorld Seen and Unseen use separate benchmark instances. ScienceWorld uses its reserved 64 tasks as Seen and the original 200 test tasks as Unseen. Complete-action evaluation uses 32 tasks per benchmark test split, excluding ALFWorld's 64 reserved tasks.

\subsection{Training Objectives and Optimization}
\label{app:training-protocol}

All warmup methods start from the same Base model. GRPO after warmup uses the second task pool and no teacher supervision. The GRPO warmup baseline also switches pools and loads its warmup weights into a new optimizer, as shown in Figs.~\ref{fig:teaser} and~\ref{fig:warmup-cross-environment}. Direct GRPO instead uses the second pool for all 100 steps without restarting, while direct OPD uses the warmup pool throughout. Extended ALFWorld comparisons run 100 or 280 GRPO steps after 40 warmup steps.

In ALFWorld 4B, SFT processes 5,120 teacher trajectories once over 40 steps, using 9,662,640 supervised response tokens. SFT-RS retains 2,375 successful trajectories and repeats them in a fixed order, processing 7.07 million tokens over 40 steps. OPW processes 9.76 million student-generated response tokens. All loss counts exclude prompts, environment observations, and padding.

In OPD and GRPO, each microbatch's response-token mean is weighted by its share of the minibatch's trajectories before gradient accumulation. In ALFWorld 4B SFT, each data-parallel rank accumulates 64 one-trajectory microbatches. Their token-normalized gradients are summed before clipping, with division by the microbatch count applied only to the logged loss. Axes count training steps, each covering one batch and its optimization. A 40-step SFT warmup followed by 60 GRPO steps contains 160 optimizer steps, compared with 200 for the OPW and GRPO-warmup paths.

\subsection{Evaluation and Statistical Summaries}
\label{app:evaluation-protocol}
\label{app:coverage}

Table~\ref{tab:evaluation-collections} lists the trajectory collections used for each analysis. Learning curves average eight rollouts per task and weight tasks equally. All eight rollouts come from the same trained checkpoint. Curves connect recorded evaluations without smoothing. A checkpoint can have several independently sampled evaluations, so each comparison uses its own initial evaluation.

\begin{table}[t]
\centering\small
\caption{Trajectory collections used for evaluation.}
\label{tab:evaluation-collections}
\setlength{\tabcolsep}{4pt}
\renewcommand{\arraystretch}{1.08}
\begin{tabular}{@{\hspace{3pt}}>{\raggedright\arraybackslash}p{122pt}c>{\raggedright\arraybackslash}p{192pt}@{\hspace{3pt}}}
\toprule
\textbf{Collection} & \textbf{Rollouts/task} & \textbf{Measurements} \\
\midrule
Original evaluations & 8 & Learning curves, final task performance and pass\textasciicircum{}8 \\
New warmup evaluations & 64 & Success coverage and repeated task success before GRPO \\
First eight of the new collection & 8 & Warmup action composition and standard OPW's initial score in supervision removal \\
ALFWorld warmup diagnostics & 32 & Validity and repeated task success on 64 reserved tasks \\
ScienceWorld warmup diagnostics & 8 & Complete success and partial-credit reward contrast on 64 Seen tasks \\
Duration comparisons & 8 & Final performance under different OPW + GRPO allocations \\
\bottomrule
\end{tabular}
\end{table}

\noindent\textbf{Early gain and average task performance.}
Let $R_t$ denote the task mean after $t$ GRPO steps following warmup, measured as task success in ALFWorld or task score in ScienceWorld on a 0--100 scale. Early gain is $\Delta R_{20}=R_{20}-R_0$. For recorded evaluations at $0=t_0<\cdots<t_m=T$, average task performance during GRPO is
\begin{equation}
\bar R_T=\frac{1}{T}\sum_{j=1}^{m}
(t_j-t_{j-1})\frac{R_{t_j}+R_{t_{j-1}}}{2}.
\label{eq:average-performance}
\end{equation}
This averages the recorded curve by the trapezoidal rule, with GRPO steps counted from the end of warmup. Full-training and duration comparisons count both stages. Target-reaching times use the first scheduled evaluation meeting the stated success target, without interpolation.

\noindent\textbf{Task-level success probabilities.}
For independent rollouts on a task $x$ with success probability $p_x$, pass@$k(x)=1-(1-p_x)^k$ measures at least one success, while pass\textasciicircum{}$k(x)=p_x^k$ measures success in all $k$ rollouts. Given $n$ rollouts with $c_x$ successes, we use~\citep{chen2021codex,yao2024taubench}
\begin{equation}
\widehat{\mathrm{pass@}k}(x)
=1-\frac{\binom{n-c_x}{k}}{\binom{n}{k}},
\qquad
\widehat{\text{pass\textasciicircum{}k}}(x)
=\frac{\binom{c_x}{k}}{\binom{n}{k}},
\label{eq:success-estimators}
\end{equation}
then average the estimates over tasks. A binomial coefficient is zero when its upper argument is below $k$. Table~\ref{tab:warmup-profile} uses $n=64$ and $k\in\{8,16,32,64\}$. At $k=n$, these are the fractions of tasks solved at least once and in every rollout. ScienceWorld complete success requires a terminal trajectory with normalized task reward one, excluding partial progress.

\section{Additional Warmup Comparisons}
\label{app:cross-environment}

\subsection{Full Warmup Results}

OPW has the highest overall repeated task success in both environments after 40 warmup and 60 GRPO steps. Table~\ref{tab:warmup-strategies-reliability} reports pass\textasciicircum{}8 by task type for the methods in Table~\ref{tab:warmup-strategies-mean}.

\begin{table*}[t]
\caption{Final repeated task success (pass\textasciicircum{}8, \%) after 40 warmup and 60 GRPO steps (Qwen3-4B).}
\label{tab:warmup-strategies-reliability}
\centering\fontsize{7.5}{9}\selectfont
\renewcommand{\arraystretch}{1.12}
\setlength{\tabcolsep}{1pt}
\arrayrulecolor{black}
\begin{tabular}{@{\hspace{3pt}}>{\raggedright\arraybackslash}p{31pt}>{\raggedright\arraybackslash}p{78pt}*{7}{>{\centering\arraybackslash}p{17.8pt}}|*{7}{>{\centering\arraybackslash}p{17.8pt}}@{\hspace{3pt}}}
\toprule
\multirow{2}{*}{\textbf{Split}} & \multirow{2}{78pt}{\diagbox[width=78pt,height=20pt]{\textbf{Warmup}}{\textbf{Task type}}} & \multicolumn{7}{c|}{\textbf{ALFWorld}} & \multicolumn{7}{c}{\textbf{ScienceWorld}} \\
\cmidrule(lr){3-9}\cmidrule(l){10-16}
& & Pick & Clean & Heat & Cool & Exam & Pick2 & All & Find & Life & Gene. & Mix & Prop. & Phys. & All \\
\midrule
\multirow{4}{*}{\textbf{Seen}} & {\fontsize{7.5}{9}\selectfont Off-policy: SFT} & 82.9 & 51.9 & \subbst{56.2} & 36.0 & 38.5 & 50.0 & 55.7 & 0.0 & \bst{25.0} & 0.0 & 0.0 & 0.0 & \subbst{9.5} & 4.7 \\
& {\fontsize{7.5}{9}\selectfont Off-policy: SFT-RS} & 91.4 & \subbst{70.4} & 50.0 & \subbst{44.0} & \subbst{69.2} & 54.2 & \subbst{65.7} & \subbst{12.5} & 0.0 & 0.0 & 0.0 & \subbst{7.4} & 4.8 & 6.2 \\
& {\fontsize{7.5}{9}\selectfont On-policy: GRPO} & \subbst{94.3} & 37.0 & 43.8 & 36.0 & 23.1 & \subbst{58.3} & 54.3 & \subbst{12.5} & 0.0 & 0.0 & 0.0 & 0.0 & \bst{19.0} & \subbst{7.8} \\
\rowcolor{OPDLightBlue}
\cellcolor{white} & {\fontsize{7.5}{9}\selectfont \textbf{OPW (Ours)}} & \bst{97.1} & \bst{92.6} & \bst{81.2} & \bst{80.0} & \bst{84.6} & \bst{83.3} & \bst{87.9} & \bst{62.5} & \bst{25.0} & \bst{50.0} & 0.0 & \bst{63.0} & \bst{19.0} & \bst{43.8} \\
\midrule
\multirow{4}{*}{\textbf{Unseen}} & {\fontsize{7.5}{9}\selectfont Off-policy: SFT} & \subbst{75.0} & 45.2 & 52.2 & 52.4 & 27.8 & 52.9 & 51.5 & \subbst{3.8} & 2.9 & \subbst{12.0} & 0.0 & 1.5 & \bst{26.3} & \subbst{6.0} \\
& {\fontsize{7.5}{9}\selectfont Off-policy: SFT-RS} & \bst{87.5} & \subbst{64.5} & \subbst{65.2} & \subbst{76.2} & \subbst{55.6} & \subbst{58.8} & \subbst{68.7} & 1.9 & \subbst{8.6} & 4.0 & 0.0 & \subbst{3.1} & 10.5 & 4.5 \\
& {\fontsize{7.5}{9}\selectfont On-policy: GRPO} & 70.8 & 29.0 & 34.8 & 28.6 & 22.2 & 17.6 & 35.1 & \bst{7.7} & 0.0 & \subbst{12.0} & 0.0 & 1.5 & \subbst{15.8} & 5.5 \\
\rowcolor{OPDLightBlue}
\cellcolor{white} & {\fontsize{7.5}{9}\selectfont \textbf{OPW (Ours)}} & \bst{87.5} & \bst{100.0} & \bst{82.6} & \bst{95.2} & \bst{77.8} & \bst{82.4} & \bst{88.8} & \subbst{3.8} & \bst{14.3} & \bst{16.0} & 0.0 & \bst{4.6} & \subbst{15.8} & \bst{8.5} \\
\bottomrule
\end{tabular}
\end{table*}

Across environments, model sizes, and splits, OPW has the highest average task performance over 60 GRPO steps (Table~\ref{tab:warmup-cross-environment-summary}). GRPO warmup yields a larger early gain on ScienceWorld 1.7B Unseen.

\begin{table}[htbp]
\caption{Early gains and average task performance during 60 GRPO steps after warmup.}
\label{tab:warmup-cross-environment-summary}
\centering\fontsize{8.5}{10}\selectfont
\setlength{\tabcolsep}{2pt}
\renewcommand{\arraystretch}{1.12}
\arrayrulecolor{black}
\begin{tabular}{@{\hspace{3pt}}>{\raggedright\arraybackslash}p{88pt}*{4}{>{\centering\arraybackslash}p{31pt}}|*{4}{>{\centering\arraybackslash}p{31pt}}@{\hspace{3pt}}}
\toprule
\multirow{3}{*}{\diagbox[width=88pt,height=32pt]{\textbf{Warmup}}{\textbf{Benchmark}}} & \multicolumn{4}{c|}{\textbf{ALFWorld}} & \multicolumn{4}{c}{\textbf{ScienceWorld}} \\
\cmidrule(lr){2-5}\cmidrule(lr){6-9}
& \multicolumn{2}{c}{Seen} & \multicolumn{2}{c|}{Unseen} & \multicolumn{2}{c}{Seen} & \multicolumn{2}{c}{Unseen} \\
\cmidrule(lr){2-3}\cmidrule(lr){4-5}\cmidrule(lr){6-7}\cmidrule(lr){8-9}
& $\Delta R_{20}$ & $\bar R_{60}$ & $\Delta R_{20}$ & $\bar R_{60}$ & $\Delta R_{20}$ & $\bar R_{60}$ & $\Delta R_{20}$ & $\bar R_{60}$ \\
\midrule
\rowcolor{TableHeader}
\multicolumn{9}{c}{\textbf{Qwen3-1.7B}} \\
{\fontsize{7.5}{9}\selectfont Off-policy: SFT} & +9.7 & 23.6 & +7.5 & 25.8 & +0.5 & 13.6 & +2.2 & 18.7 \\
{\fontsize{7.5}{9}\selectfont Off-policy: SFT-RS} & +6.8 & 20.7 & +5.0 & 21.1 & +3.2 & 12.9 & +4.6 & 19.4 \\
{\fontsize{7.5}{9}\selectfont On-policy: GRPO} & +7.8 & 37.2 & +6.8 & 37.8 & +13.7 & 28.1 & \textbf{+14.0} & 34.1 \\
\rowcolor{OPDLightBlue}
{\fontsize{7.5}{9}\selectfont \textbf{OPW (Ours)}} & \textbf{+26.9} & \textbf{73.7} & \textbf{+19.6} & \textbf{71.9} & \textbf{+14.2} & \textbf{62.0} & +5.6 & \textbf{58.1} \\
\midrule
\rowcolor{TableHeader}
\multicolumn{9}{c}{\textbf{Qwen3-4B}} \\
{\fontsize{7.5}{9}\selectfont Off-policy: SFT} & +14.1 & 56.5 & +15.0 & 53.8 & +15.9 & 43.5 & +17.1 & 51.4 \\
{\fontsize{7.5}{9}\selectfont Off-policy: SFT-RS} & +8.7 & 56.3 & +7.3 & 54.5 & \textbf{+22.8} & 47.4 & \textbf{+19.5} & 53.1 \\
{\fontsize{7.5}{9}\selectfont On-policy: GRPO} & +11.4 & 69.5 & +11.9 & 65.0 & +0.9 & 57.0 & +0.2 & 59.4 \\
\rowcolor{OPDLightBlue}
{\fontsize{7.5}{9}\selectfont \textbf{OPW (Ours)}} & \textbf{+41.8} & \textbf{84.7} & \textbf{+31.4} & \textbf{86.6} & +22.6 & \textbf{78.9} & +3.5 & \textbf{65.2} \\
\bottomrule
\end{tabular}
\end{table}

\subsection{Sensitivity to Training Length and Seed}
\label{app:shorter-warmup}
\label{app:longer-rlvr}
\label{app:training-dynamics}
\label{app:training-seeds}

\noindent\textbf{Shorter warmup.}
With 20 warmup and 80 GRPO steps, OPW retains the highest final ALFWorld 4B success and pass\textasciicircum{}8 on both splits (Table~\ref{tab:warmup-task-types-20}). Fig.~\ref{fig:warmup-cross-environment-20} uses the main comparison's task sets and Base evaluations, counting total steps. Direct OPD is evaluated every 20 steps through step 100.

\begin{table}[htbp]
\centering\small
\caption{Final task performance and repeated task success for 20-step warmup (ALFWorld 4B).}
\label{tab:warmup-task-types-20}
\setlength{\tabcolsep}{6pt}
\renewcommand{\arraystretch}{1.12}
\begin{tabular}{@{\hspace{6pt}}lcc|cc@{\hspace{6pt}}}
\toprule
\multirow{2}{*}{\textbf{Warmup}} & \multicolumn{2}{c|}{\textbf{Seen}} & \multicolumn{2}{c}{\textbf{Unseen}} \\
\cmidrule(lr){2-3}\cmidrule(lr){4-5}
& pass@1 & pass\textasciicircum{}8 & pass@1 & pass\textasciicircum{}8 \\
\midrule
SFT & \subbst{91.2} & \subbst{72.1} & 83.3 & \subbst{63.4} \\
SFT-RS & 89.1 & 68.6 & \subbst{83.5} & 59.0 \\
GRPO & 72.3 & 35.0 & 69.9 & 26.1 \\
\rowcolor{OPDLightBlue}
\textbf{OPW (Ours)} & \bst{95.7} & \bst{85.0} & \bst{93.3} & \bst{76.9} \\
\bottomrule
\end{tabular}
\end{table}

\begin{figure}[t]
\centering
\includegraphics[width=\textwidth]{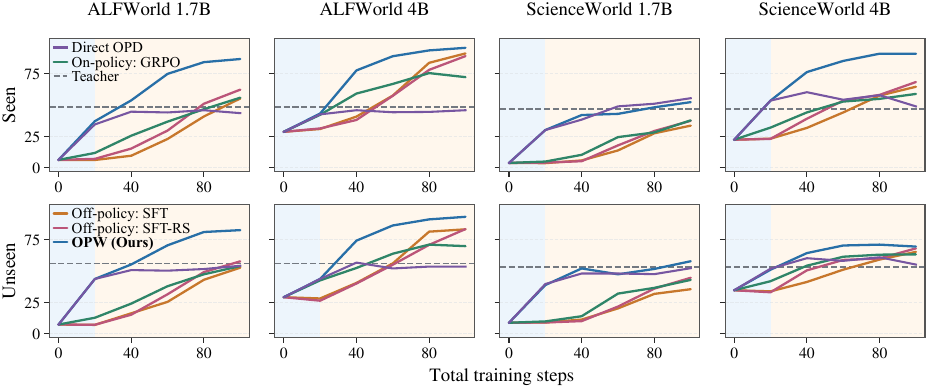}
\caption{Task performance during 20 warmup and 80 GRPO steps.}
\label{fig:warmup-cross-environment-20}
\end{figure}

\noindent\textbf{Longer GRPO training.}
After 40 warmup and 100 GRPO steps, SFT-RS approaches OPW's final Seen success (95.2\% versus 96.2\%), but its average GRPO success remains lower (71.0\% versus 89.4\%). Unseen averages are 68.3\% and 90.9\%, respectively (Table~\ref{tab:warmup-learning}). Fig.~\ref{fig:warmup-long-training} extends GRPO to 280 steps and shows sustained high success after OPW.

\begin{figure*}[t]
\centering
\includegraphics[width=\textwidth]{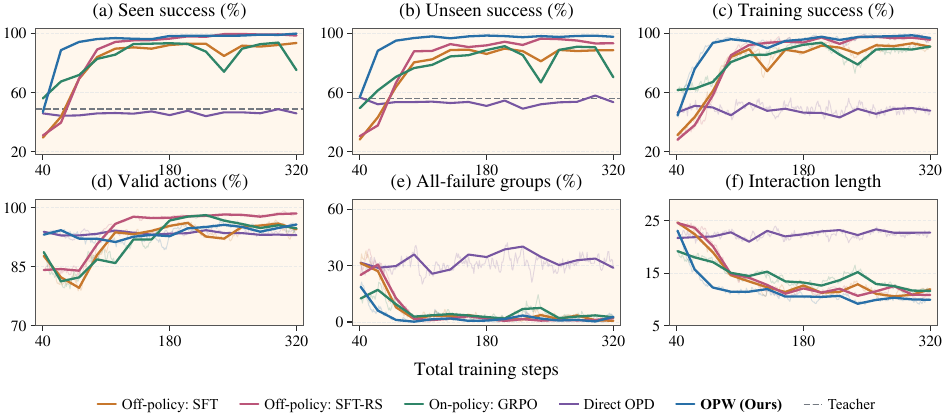}
\caption{Task performance (a,b) and training dynamics (c--f) during extended training.}
\label{fig:warmup-long-training}
\end{figure*}

\begin{table}[htbp]
\caption{Average and final task performance over 100 GRPO steps after warmup (ALFWorld 4B).}
\label{tab:warmup-learning}
\par\smallskip
\centering\fontsize{8.5}{10}\selectfont
\setlength{\tabcolsep}{2pt}
\renewcommand{\arraystretch}{1.12}
\begin{tabular}{@{\hspace{3pt}}>{\raggedright\arraybackslash}p{88pt}*{2}{>{\centering\arraybackslash}p{34pt}}|*{2}{>{\centering\arraybackslash}p{34pt}}@{\hspace{3pt}}}
\toprule
\multirow{2}{88pt}{\diagbox[width=88pt,height=20pt]{\textbf{Warmup}}{\textbf{Task set}}} & \multicolumn{2}{c|}{Seen} & \multicolumn{2}{c}{Unseen} \\
\cmidrule(lr){2-3}\cmidrule(lr){4-5}
& $\bar R_{100}$ & $R_{100}$ & $\bar R_{100}$ & $R_{100}$ \\
\midrule
{\fontsize{7.5}{9}\selectfont Off-policy: SFT} & 69.3 & 90.4 & 65.6 & 88.6 \\
{\fontsize{7.5}{9}\selectfont Off-policy: SFT-RS} & 71.0 & 95.2 & 68.3 & 92.6 \\
{\fontsize{7.5}{9}\selectfont On-policy: GRPO} & 76.3 & 92.7 & 70.8 & 84.3 \\
\rowcolor{OPDLightBlue}
{\fontsize{7.5}{9}\selectfont \textbf{OPW (Ours)}} & \textbf{89.4} & \textbf{96.2} & \textbf{90.9} & \textbf{96.5} \\
\bottomrule
\end{tabular}
\end{table}

In Fig.~\ref{fig:warmup-long-training}, test panels (a,b) use unsmoothed evaluations every 20 steps, with eight rollouts on 140 Seen and 134 Unseen tasks. Training panels (c)--(f) use 16 tasks with eight rollouts per step. Validity pools recorded actions, reward-group fractions average task groups, and interaction length counts action turns. Training curves use exponential moving averages with coefficient 0.2, while interval means weight raw steps equally. Direct OPD remains on the warmup pool over the same total steps.

OPW's all-success group fraction rises from 44.7\% to 72.8\% between the first two 20-step GRPO intervals. These groups indicate reliable completion but have zero direct GRPO advantages. Over the final 100 steps, OPW averages 97.3\% training success with 1.6\% all-failure groups, compared with 86.7\% and 4.0\% after GRPO warmup.

\noindent\textbf{Second training seed.}
Both ALFWorld 4B seeds use 40 warmup and 100 GRPO steps, with eight rollouts per task on both test splits. Table~\ref{tab:warmup-seeds} reports OPW minus GRPO warmup for initial success $R_0$, early gain $\Delta R_{20}$, and average success $\bar R_{100}$. Under both seeds, OPW starts lower on Seen and higher on Unseen, then achieves larger early gains and higher averages on both splits.

\begin{table}[htbp]
\caption{OPW minus GRPO warmup under two training seeds (percentage points).}
\label{tab:warmup-seeds}
\centering\small
\setlength{\tabcolsep}{6pt}
\renewcommand{\arraystretch}{1.12}
\begin{tabular}{@{\hspace{3pt}}>{\centering\arraybackslash}p{28pt}*{3}{>{\centering\arraybackslash}p{44pt}}|*{3}{>{\centering\arraybackslash}p{44pt}}@{\hspace{3pt}}}
\toprule
\multirow{2}{*}{\textbf{Seed}} & \multicolumn{3}{c|}{\textbf{Seen}} & \multicolumn{3}{c}{\textbf{Unseen}} \\
\cmidrule(lr){2-4}\cmidrule(lr){5-7}
& $R_0$ & $\Delta R_{20}$ & $\bar R_{100}$ & $R_0$ & $\Delta R_{20}$ & $\bar R_{100}$ \\
\midrule
1 & -9.2 & +30.4 & +13.0 & +7.3 & +19.5 & +20.1 \\
2 & -0.8 & +17.2 & +6.6 & +6.4 & +9.0 & +9.5 \\
\bottomrule
\end{tabular}
\end{table}

\subsection{Training Computation and Target-Reaching Cost}
\label{app:training-cost}

We multiply recorded training times by eight GPUs for each path with 40 warmup and 60 GRPO steps in Fig.~\ref{fig:warmup-cross-environment}. Timers include trajectory generation, environment interaction, teacher scoring, and student optimization. SFT and SFT-RS include 10.9 GPU-hours for collecting 5,120 shared teacher trajectories before filtering. Separate evaluations, research diagnostics, checkpoint saving, and startup are excluded. Student training and teacher scoring share the eight-GPU allocation, with fully sharded data parallelism, sequence-parallel size four, and data-parallel size two. Supplementary runs use L20X GPUs. Unknown GPU models in earlier runs prevent hardware normalization.

\begin{table}[htbp]
\centering
\caption{Training cost and cost to reach target Seen success in ALFWorld (GPU-hours).}
\label{tab:training-cost-targets}
\label{tab:training-cost}
\fontsize{8.5}{10.5}\selectfont
\setlength{\tabcolsep}{3pt}
\renewcommand{\arraystretch}{1.08}
\begin{tabular}{@{\hspace{3pt}}>{\raggedright\arraybackslash}p{100pt}*{6}{>{\centering\arraybackslash}p{38pt}}@{\hspace{3pt}}}
\toprule
\multirow{2}{100pt}{\diagbox[width=100pt,height=20pt]{\textbf{Warmup}}{\textbf{Cost}}} & \multicolumn{4}{c}{100 training steps} & \multicolumn{2}{c}{To reach} \\
\cmidrule(lr){2-5}\cmidrule(lr){6-7}
& \shortstack{Data\\collection} & \textbf{Warmup} & GRPO & Total & $>80\%$ & $>90\%$ \\
\midrule
\rowcolor{TableHeader}
\multicolumn{7}{c}{\textbf{Qwen3-1.7B}} \\
{\fontsize{7.5}{9}\selectfont Off-policy: SFT} & 10.9 & 2.5 & 14.1 & 27.5 & -- & -- \\
{\fontsize{7.5}{9}\selectfont Off-policy: SFT-RS} & 10.9 & 1.9 & 13.9 & 26.7 & -- & -- \\
{\fontsize{7.5}{9}\selectfont On-policy: GRPO} & 0.0 & 9.3 & 12.9 & \textbf{22.2} & -- & -- \\
\rowcolor{OPDLightBlue}
{\fontsize{7.5}{9}\selectfont \textbf{OPW (Ours)}} & 0.0 & 12.8 & 11.1 & 23.9 & \textbf{20.3} & -- \\
\midrule
\rowcolor{TableHeader}
\multicolumn{7}{c}{\textbf{Qwen3-4B}} \\
{\fontsize{7.5}{9}\selectfont Off-policy: SFT} & 10.9 & 4.8 & 15.7 & 31.4 & 31.4 & -- \\
{\fontsize{7.5}{9}\selectfont Off-policy: SFT-RS} & 10.9 & 3.2 & 16.8 & 30.9 & 30.9 & -- \\
{\fontsize{7.5}{9}\selectfont On-policy: GRPO} & 0.0 & 13.0 & 16.6 & 29.6 & 29.6 & -- \\
\rowcolor{OPDLightBlue}
{\fontsize{7.5}{9}\selectfont \textbf{OPW (Ours)}} & 0.0 & 14.9 & 12.4 & \textbf{27.3} & \textbf{19.6} & \textbf{23.6} \\
\bottomrule
\multicolumn{7}{@{\hspace{3pt}}l@{}}{\fontsize{7.5}{9}\selectfont --: target not reached within 100 steps.}
\end{tabular}
\end{table}

OPW reduces total cost relative to GRPO warmup at 4B (27.3 versus 29.6 GPU-hours) but increases it at 1.7B (23.9 versus 22.2). Table~\ref{tab:training-cost-targets} uses the first evaluation at total step 40, 60, 80, or 100 exceeding the target, without interpolation. At 4B, OPW reaches both targets at the same evaluations on Seen and Unseen. GRPO warmup exceeds 80\% only on Seen. At 1.7B, only OPW exceeds 80\%, at step 80 on Seen and step 100 on Unseen. Only 4B OPW exceeds 90\% within 100 steps.

\section{Teacher Supervision on Student-Generated Trajectories}
\label{app:supervision}

\subsection{Teacher-Target Probabilities after One Optimizer Step}
\label{app:token-update}

From Base, we take one AdamW step with learning rate $10^{-6}$, using 128 student trajectories for OPD/GRPO or 128 teacher trajectories for SFT. These batches contain 329,646 and 257,068 response tokens, respectively. OPD applies $\ell_t=-\log\pi_\theta(v_t\mid h_t)$ at student token histories $h_t$, with teacher top-1 targets $v_t$. SFT uses this loss with teacher tokens on teacher trajectories. We score log-probability changes at common histories where teacher targets disagree with student tokens.

The batch has 16 tasks and 22,279 disagreements. Eight groups contain only failures, six have mixed rewards, and two contain only successes. Table~\ref{tab:one-update-corrections} averages changes over all disagreements, 11,839 positions in 64 all-failure trajectories, and 4,455 positions with the lowest Base probability for the teacher target. OPD and GRPO use all 16 task groups in one minibatch, compared with eight per minibatch in main training. In all-failure groups, SFT changes probabilities through transfer from teacher trajectories, while OPD supervises the student histories directly. GRPO advantages are zero in these groups, but other groups can change their probabilities through shared parameters. Across the batch, disagreements occupy 6.8\% of response positions but account for 97.7\% of teacher-target loss and 95.2\% of the positive log-probability change after OPD.

\begin{table}[htbp]
\centering\small
\caption{Teacher-target log-probability changes after one optimizer step (ALFWorld 4B).}
\label{tab:one-update-corrections}
\setlength{\tabcolsep}{7pt}
\renewcommand{\arraystretch}{1.12}
\begin{tabular}{@{\hspace{7pt}}lccc@{\hspace{7pt}}}
\toprule
\textbf{Objective} & \makecell{All\\disagreements} & \makecell{All-failure\\groups} & \makecell{Lowest-probability\\20\%} \\
\midrule
SFT & $+0.1279$ & $+0.128$ & $+0.290$ \\
GRPO & $-0.0004$ & $+0.004$ & $+0.017$ \\
\rowcolor{OPDLightBlue}
OPD & $\mathbf{+0.1812}$ & $\mathbf{+0.183}$ & $\mathbf{+0.392}$ \\
\bottomrule
\end{tabular}
\end{table}

\subsection{Resampling Student Trajectories}
\label{app:student-resampling}

Both ALFWorld Qwen3-4B conditions use teacher top-1 cross-entropy for 40 warmup steps, followed by 60 teacher-free GRPO steps. The fixed version reuses initial-student trajectories while recomputing teacher targets and current student probabilities. OPW resamples trajectories as the student changes. Evaluations use eight rollouts per task on 140 Seen and 134 Unseen tasks. Table~\ref{tab:student-resampling-checkpoints} reports the evaluations underlying Table~\ref{tab:student-resampling-summary}. Resampling improves early gains and average success, while fixed trajectories yield higher final Seen success. Resampling jointly changes histories, diversity, and task coverage.

\begin{table}[htbp]
\centering\small
\caption{Task performance after warmup on fixed or resampled trajectories (ALFWorld 4B).}
\label{tab:student-resampling-checkpoints}
\setlength{\tabcolsep}{6pt}
\renewcommand{\arraystretch}{1.12}
\begin{tabular}{@{\hspace{6pt}}llcccc@{\hspace{6pt}}}
\toprule
\multirow{2}{*}{\textbf{Task set}} & \multirow{2}{*}{\textbf{Trajectories}} & \multicolumn{4}{c}{\textbf{GRPO steps after warmup}} \\
\cmidrule(lr){3-6}
& & 0 & 20 & 40 & 60 \\
\midrule
\textbf{Seen} & Fixed & 41.1 & 79.6 & 94.4 & 97.4 \\
\rowcolor{OPDLightBlue}
\cellcolor{white} & \textbf{OPW (Ours)} & 46.8 & 88.6 & 94.1 & 96.1 \\
\midrule
\textbf{Unseen} & Fixed & 50.5 & 79.7 & 91.7 & 95.1 \\
\rowcolor{OPDLightBlue}
\cellcolor{white} & \textbf{OPW (Ours)} & 56.7 & 88.2 & 95.0 & 96.7 \\
\bottomrule
\end{tabular}
\end{table}

\subsection{Removing Supervision at Repeated Invalid Actions}
\label{app:error-supervision-removal}

Two ALFWorld Qwen3-4B runs share OPW's initialization, tasks, teacher, batch size, seed, and optimizer for 40 warmup and 60 GRPO steps. Targeted removal zeros teacher-loss weights on action-content tokens when the action is invalid and has appeared earlier in the trajectory. Inputs, targets, and loss normalization stay unchanged. Random removal selects complete actions with the same token-length counts as repeated invalid actions in its current batch, allowing overlap. Branches generate their own trajectories, so total removal counts can differ. Targeted removal excludes 34,536 tokens in 6,489 actions (0.354\% of response tokens), versus 32,132 tokens in 5,992 actions (0.330\%) under random removal. Random removal includes 7,199 tokens in 1,474 repeated invalid actions.

Targeted and random removal yield similar warmup validity on Seen (93.8\% versus 93.4\%) and Unseen (93.5\% versus 93.7\%). Yet targeted removal produces smaller early GRPO gains on both splits (Table~\ref{tab:error-supervision-learning}), supporting the contribution of supervision at repeated errors. All conditions use eight rollouts per task on 140 Seen and 134 Unseen tasks at 0, 20, 40, and 60 GRPO steps. OPW's initial evaluation uses the first eight of the new 64-rollout collection, explaining differences from the main summary. Later OPW evaluations use the original collections. The table reports early gain $\Delta R_{20}$ in percentage points, average success $\bar R_{60}$ from Eq.~\ref{eq:average-performance}, and final success $R_{60}$ in percent.

\begin{table}[htbp]
\caption{GRPO learning after targeted and random supervision removal (ALFWorld 4B).}
\label{tab:error-supervision-learning}
\par\smallskip
\centering\fontsize{8.5}{10}\selectfont
\setlength{\tabcolsep}{2pt}
\renewcommand{\arraystretch}{1.12}
\begin{tabular}{@{\hspace{3pt}}>{\raggedright\arraybackslash}p{88pt}*{3}{>{\centering\arraybackslash}p{33pt}}|*{3}{>{\centering\arraybackslash}p{33pt}}@{\hspace{3pt}}}
\toprule
\multirow{2}{88pt}{\diagbox[width=88pt,height=20pt]{\textbf{Warmup}}{\textbf{Task set}}} & \multicolumn{3}{c|}{Seen} & \multicolumn{3}{c}{Unseen} \\
\cmidrule(lr){2-4}\cmidrule(lr){5-7}
& $\Delta R_{20}$ & $\bar R_{60}$ & $R_{60}$ & $\Delta R_{20}$ & $\bar R_{60}$ & $R_{60}$ \\
\midrule
\rowcolor{OPDLightBlue}
{\fontsize{7.5}{9}\selectfont \textbf{OPW (Ours)}} & 41.5 & 84.7 & 96.1 & 32.2 & \textbf{86.5} & \textbf{96.7} \\
Random removal & \textbf{42.0} & \textbf{84.8} & \textbf{97.1} & \textbf{35.2} & 82.9 & 94.3 \\
Targeted removal & 34.6 & 80.1 & 95.4 & 29.2 & 80.9 & 92.9 \\
\bottomrule
\end{tabular}
\end{table}

\section{From Teacher Targets to Action Value}
\label{app:action-value-analysis}

\subsection{Teacher Alternatives, Task Value, and Retention}
\label{app:retained-alternatives}
\label{app:expanded-action-value}

In ALFWorld Qwen3-4B, we test 311 action-token alternatives from 97 tasks. We inspect eight Base rollouts per task from the first 128 distinct training tasks in trajectory order. Eligible positions have Base probability of at least 0.9 for the original token and a different teacher top-1 token, with both tokens containing letters or digits. We retain 21 positions from an eight-task analysis and add up to four distinct token histories per task, without using continuation outcomes or trained-policy probabilities for selection. At each position, we choose either token and let Base complete the action and remaining trajectory. Four continuations per choice yield 2,488 rollouts. We average continuations within positions, positions within tasks, and tasks equally. Teacher alternatives raise action validity from 67.44\% to 82.22\% and continuation success from 26.48\% to 29.70\%. On the 89 added tasks, continuation success still improves by 2.39 percentage points.

We track teacher-target probabilities at these action positions and on a separate set of 39,490 disagreements from 256 student trajectories across 32 tasks (Table~\ref{tab:warmup-choice-persistence}). All checkpoints use the same token histories and teacher top-1 targets. Probabilities use the full softmax before temperature scaling or top-$p$ truncation, averaged within tasks and then equally across tasks. Over 40 teacher-free GRPO steps, OPW's mean probability rises from 39.3\% to 47.2\% at the selected action positions. On the broader set, it remains near its warmup level (43.8\% to 43.2\%) and highest on all 32 tasks. Selected action positions thus show further gains, while the broader set shows retention. The table reports warmup probabilities $P_0$ in percent and changes $\Delta P_t=P_t-P_0$ in percentage points.

\begin{table}[htbp]
\caption{Teacher-target probabilities during teacher-free GRPO (ALFWorld 4B).}
\label{tab:warmup-choice-persistence}
\par\smallskip
\centering\fontsize{8.5}{10}\selectfont
\setlength{\tabcolsep}{2pt}
\renewcommand{\arraystretch}{1.12}
\begin{tabular}{@{\hspace{3pt}}>{\raggedright\arraybackslash}p{88pt}*{3}{>{\centering\arraybackslash}p{34pt}}|*{3}{>{\centering\arraybackslash}p{34pt}}@{\hspace{3pt}}}
\toprule
\multirow{2}{88pt}{\diagbox[width=88pt,height=20pt]{\textbf{Warmup}}{\textbf{Positions}}} & \multicolumn{3}{c|}{All disagreements} & \multicolumn{3}{c}{Action positions} \\
\cmidrule(lr){2-4}\cmidrule(lr){5-7}
& $P_0$ & $\Delta P_{20}$ & $\Delta P_{40}$ & $P_0$ & $\Delta P_{20}$ & $\Delta P_{40}$ \\
\midrule
{\fontsize{7.5}{9}\selectfont Off-policy: SFT} & 17.7 & $+0.3$ & $+1.2$ & 3.8 & $+2.8$ & $+6.2$ \\
{\fontsize{7.5}{9}\selectfont Off-policy: SFT-RS} & 17.9 & $-1.4$ & $+1.0$ & 3.7 & $+4.7$ & $+14.8$ \\
{\fontsize{7.5}{9}\selectfont On-policy: GRPO} & 17.7 & $+1.6$ & $-0.9$ & 12.5 & $+6.2$ & $+8.3$ \\
\rowcolor{OPDLightBlue}
{\fontsize{7.5}{9}\selectfont \textbf{OPW (Ours)}} & \textbf{43.8} & $-0.2$ & $-0.6$ & \textbf{39.3} & $+6.1$ & $+7.9$ \\
\bottomrule
\end{tabular}
\end{table}

\subsection{Complete-Action Value at Common Histories}
\label{app:complete-action-value}

We take one trajectory per warmup method from each of 32 Seen and 32 Unseen tasks and inspect positions after four and twelve actions. Checkpoints share 512 records, including 447 active positions from 63 tasks and 65 records from already-ended trajectories. Each policy samples four actions per active position. Base supplies four continuations per distinct position--action pair, reusing continuations for identical actions at their sampled frequencies. We average positions and sources within tasks, then weight tasks equally. Ended records retain their outcomes only in all-position success.

We replay actions in a reset game, checking observations and admissible actions. Base receives only the sampled \texttt{<action>} field. Actions and continuations use temperature 0.7, top-$p=0.95$, a 1,024-token response limit, a 32,768-token context limit, and at most 30 interaction turns. Table~\ref{tab:sampled-action-value} reports validity $V_t$ and continuation success $C_t$ (\%) after $t$ GRPO steps. Gains are in percentage points, with $\Delta C_{40}=C_{40}-C_0$.

\begin{table}[htbp]
\caption{Complete-action validity and continuation success under Base (ALFWorld 4B).}
\label{tab:sampled-action-value}
\centering\fontsize{8.5}{10}\selectfont
\setlength{\tabcolsep}{2pt}
\renewcommand{\arraystretch}{1.12}
\begin{tabular}{@{\hspace{3pt}}>{\raggedright\arraybackslash}p{88pt}*{4}{>{\centering\arraybackslash}p{30pt}}|*{4}{>{\centering\arraybackslash}p{30pt}}@{\hspace{3pt}}}
\toprule
\multirow{3}{88pt}{\diagbox[width=88pt,height=32pt]{\textbf{Warmup}}{\textbf{Metric}}} & \multicolumn{4}{c|}{Active positions} & \multicolumn{4}{c}{All positions: success} \\
\cmidrule(lr){2-5}\cmidrule(lr){6-9}
& \multicolumn{2}{c}{Validity} & \multicolumn{2}{c|}{Success} & \multicolumn{2}{c}{Seen} & \multicolumn{2}{c}{Unseen} \\
\cmidrule(lr){2-3}\cmidrule(lr){4-5}\cmidrule(lr){6-7}\cmidrule(lr){8-9}
& $V_0$ & $V_{40}$ & $C_0$ & $C_{40}$ & $C_0$ & $\Delta C_{40}$ & $C_0$ & $\Delta C_{40}$ \\
\midrule
{\fontsize{7.5}{9}\selectfont Off-policy: SFT} & 92.57 & 93.68 & 44.01 & 45.16 & 41.89 & +0.61 & 51.56 & +1.20 \\
{\fontsize{7.5}{9}\selectfont Off-policy: SFT-RS} & 92.64 & 93.74 & 44.40 & 45.56 & 42.33 & +0.42 & 51.78 & +1.34 \\
{\fontsize{7.5}{9}\selectfont On-policy: GRPO} & 93.92 & 93.78 & \textbf{45.93} & 46.47 & \textbf{43.60} & +0.05 & \textbf{52.93} & +0.93 \\
\rowcolor{OPDLightBlue}
{\fontsize{7.5}{9}\selectfont \textbf{OPW (Ours)}} & \textbf{94.45} & \textbf{94.61} & 45.00 & \textbf{46.75} & 42.26 & \textbf{+1.68} & 52.61 & \textbf{+1.51} \\
\bottomrule
\end{tabular}
\end{table}

At active positions, OPW starts with the highest validity (94.45\%) but lower continuation success than GRPO warmup (45.00\% versus 45.93\%). Over 40 GRPO steps, its continuation success rises to 46.75\%, compared with 46.47\% after GRPO warmup. Including already-ended records gives the same ordering of gains, at 1.60 versus 0.49 points, with similar final means. Since Base supplies every continuation, these gains reflect improved action selection at the evaluated histories.

\section{Behavior after Warmup and during GRPO}
\label{app:behavior}

\subsection{Execution Errors and Repetition}
\label{app:error-events}
\label{app:behavior-cross-environment}

The Qwen3-4B analysis uses the first eight trajectories per task on 140 Seen and 134 Unseen ALFWorld tasks and 64 Seen and 200 Unseen ScienceWorld tasks. Base and Teacher have one evaluation each. The four warmup methods are evaluated after warmup and after 20, 40, and 60 GRPO steps. Table~\ref{tab:warmup-action-composition} uses the warmup evaluations from the new 64-rollout collection. Tasks receive equal weight.

We lowercase action text and normalize whitespace. An action repeats if it appeared earlier, whether or not that occurrence was valid. Validity requires membership in the current valid-action list. ScienceWorld navigation aliases absent from that list count as invalid even if execution succeeds. Empty or unparseable actions are invalid format failures. Empty actions are excluded from repetition. Consecutive identical invalid actions require both adjacent actions to be invalid and identical, forming a subset of repeated invalid actions. Event shares divide by all actions in each task's eight trajectories. Error incidence is the fraction of trajectories containing an invalid action.

On ALFWorld Seen, similar fractions of SFT and OPW trajectories contain invalid actions (50.9\% and 50.0\%), but consecutive identical invalid actions occupy 3.93\% and 0.66\% of all actions. On ScienceWorld Unseen, the consecutive-error share is 13.83\% for SFT and 4.17\% for OPW. Appendix~\ref{app:trajectory-cases} illustrates these patterns with complete action sequences.

Table~\ref{tab:behavior-cross-environment} compares both model sizes using the original eight-rollout evaluations. Its initial means can differ slightly from the new warmup collection. $V_0$, $L_0$, and $R_0$ denote warmup validity, repetition, and task performance, with $\Delta X=X_{60}-X_0$. Repetition includes valid actions, and ScienceWorld task scores include partial progress. For ALFWorld 4B, OPW's validity rises by only 0.9 points during GRPO while success rises by 49.3, showing continued progress after most actions are valid.

\begingroup
\fontsize{8.5}{10}\selectfont
\setlength{\tabcolsep}{1.5pt}
\renewcommand{\arraystretch}{1.12}
\setlength{\LTcapwidth}{\linewidth}
\begin{longtable}{@{\hspace{3pt}}>{\raggedright\arraybackslash}p{80pt}*{6}{>{\centering\arraybackslash}p{22.5pt}}|*{6}{>{\centering\arraybackslash}p{22.5pt}}@{\hspace{3pt}}}
\caption{\small Execution behavior and task performance before and after 60 GRPO steps.}
\label{tab:behavior-cross-environment} \\
\toprule
\multirow{3}{80pt}{\textbf{Warmup}} & \multicolumn{6}{c|}{\textbf{Qwen3-1.7B}} & \multicolumn{6}{c}{\textbf{Qwen3-4B}} \\
\cmidrule(lr){2-7}\cmidrule(lr){8-13}
& \multicolumn{2}{c}{\makecell{Action\\validity}} & \multicolumn{2}{c}{\makecell{Repeated\\actions}} & \multicolumn{2}{c|}{\makecell{Task\\performance}} & \multicolumn{2}{c}{\makecell{Action\\validity}} & \multicolumn{2}{c}{\makecell{Repeated\\actions}} & \multicolumn{2}{c}{\makecell{Task\\performance}} \\
\cmidrule(lr){2-3}\cmidrule(lr){4-5}\cmidrule(lr){6-7}\cmidrule(lr){8-9}\cmidrule(lr){10-11}\cmidrule(lr){12-13}
& $V_0$ & $\Delta V$ & $L_0$ & $\Delta L$ & $R_0$ & $\Delta R$ & $V_0$ & $\Delta V$ & $L_0$ & $\Delta L$ & $R_0$ & $\Delta R$ \\
\midrule
\endfirsthead
\multicolumn{13}{c}{\tablename~\thetable\ (continued)} \\
\toprule
\multirow{3}{80pt}{\textbf{Warmup}} & \multicolumn{6}{c|}{\textbf{Qwen3-1.7B}} & \multicolumn{6}{c}{\textbf{Qwen3-4B}} \\
\cmidrule(lr){2-7}\cmidrule(lr){8-13}
& \multicolumn{2}{c}{\makecell{Action\\validity}} & \multicolumn{2}{c}{\makecell{Repeated\\actions}} & \multicolumn{2}{c|}{\makecell{Task\\performance}} & \multicolumn{2}{c}{\makecell{Action\\validity}} & \multicolumn{2}{c}{\makecell{Repeated\\actions}} & \multicolumn{2}{c}{\makecell{Task\\performance}} \\
\cmidrule(lr){2-3}\cmidrule(lr){4-5}\cmidrule(lr){6-7}\cmidrule(lr){8-9}\cmidrule(lr){10-11}\cmidrule(lr){12-13}
& $V_0$ & $\Delta V$ & $L_0$ & $\Delta L$ & $R_0$ & $\Delta R$ & $V_0$ & $\Delta V$ & $L_0$ & $\Delta L$ & $R_0$ & $\Delta R$ \\
\midrule
\endhead
\bottomrule
\endfoot
\bottomrule
\endlastfoot
\rowcolor{TableHeader}
\multicolumn{13}{c}{\textbf{ALFWorld, Seen}} \\*
{\fontsize{7.5}{9}\selectfont Off-policy: SFT} & 76.7 & $+6.5$ & 60.1 & $-35.9$ & 6.1 & $+39.2$ & 87.4 & $+5.4$ & 29.3 & $-9.6$ & 29.6 & $+54.4$ \\*
{\fontsize{7.5}{9}\selectfont Off-policy: SFT-RS} & 77.1 & $+4.0$ & 61.7 & $-29.5$ & 6.2 & $+37.6$ & 88.0 & $+7.5$ & 29.7 & $-16.1$ & 31.0 & $+58.0$ \\*
{\fontsize{7.5}{9}\selectfont On-policy: GRPO} & 71.1 & $+12.7$ & 38.5 & $-10.6$ & 20.1 & $+45.4$ & 90.1 & $-0.5$ & 22.7 & $-5.1$ & 56.0 & $+26.7$ \\*
\rowcolor{OPDLightBlue}
{\fontsize{7.5}{9}\selectfont \textbf{OPW (Ours)}} & 92.0 & $+4.5$ & 18.5 & $+1.6$ & 46.5 & $+39.9$ & 93.9 & $+0.9$ & 16.4 & $-3.5$ & 46.8 & $+49.3$ \\
\midrule
\rowcolor{TableHeader}
\multicolumn{13}{c}{\textbf{ScienceWorld, Seen}} \\*
{\fontsize{7.5}{9}\selectfont Off-policy: SFT} & 62.8 & $-39.6$ & 87.8 & $-39.2$ & 4.0 & $+29.4$ & 63.5 & $-23.4$ & 67.5 & $-31.0$ & 21.8 & $+38.7$ \\*
{\fontsize{7.5}{9}\selectfont Off-policy: SFT-RS} & 61.5 & $-19.2$ & 89.5 & $-28.4$ & 3.3 & $+25.0$ & 62.3 & $-13.1$ & 67.8 & $-36.0$ & 21.7 & $+42.9$ \\*
{\fontsize{7.5}{9}\selectfont On-policy: GRPO} & 35.0 & $-2.6$ & 79.9 & $-29.0$ & 9.0 & $+32.6$ & 49.4 & $+2.4$ & 41.7 & $-1.4$ & 49.8 & $+19.0$ \\*
\rowcolor{OPDLightBlue}
{\fontsize{7.5}{9}\selectfont \textbf{OPW (Ours)}} & 43.1 & $+3.7$ & 53.8 & $-17.8$ & 44.1 & $+27.6$ & 51.7 & $+2.3$ & 38.5 & $-11.3$ & 57.3 & $+29.9$ \\
\midrule
\rowcolor{TableHeader}
\multicolumn{13}{c}{\textbf{ScienceWorld, Unseen}} \\*
{\fontsize{7.5}{9}\selectfont Off-policy: SFT} & 61.0 & $-31.4$ & 88.3 & $-41.4$ & 8.8 & $+31.6$ & 57.1 & $-13.3$ & 62.6 & $-27.0$ & 32.2 & $+29.7$ \\*
{\fontsize{7.5}{9}\selectfont Off-policy: SFT-RS} & 59.8 & $-16.3$ & 86.6 & $-27.6$ & 8.5 & $+27.1$ & 56.4 & $-8.3$ & 59.1 & $-30.3$ & 34.2 & $+29.3$ \\*
{\fontsize{7.5}{9}\selectfont On-policy: GRPO} & 36.8 & $-0.6$ & 77.1 & $-24.0$ & 15.0 & $+30.5$ & 48.5 & $+4.2$ & 38.5 & $-2.1$ & 56.8 & $+2.8$ \\*
\rowcolor{OPDLightBlue}
{\fontsize{7.5}{9}\selectfont \textbf{OPW (Ours)}} & 42.5 & $+3.4$ & 44.8 & $-8.0$ & 51.8 & $+11.2$ & 50.3 & $+3.0$ & 36.6 & $-3.6$ & 61.6 & $+6.6$ \\
\end{longtable}
\endgroup

\subsection{Successful and Failure-Only Action Sequences}
\label{app:sequence-dynamics}

We evaluate both model sizes on 140 ALFWorld Seen and 200 ScienceWorld Unseen tasks, with eight rollouts per task after warmup and after 20, 40, and 60 GRPO steps. After lowercasing and whitespace normalization, we remove invalid actions and \texttt{look}, \texttt{look around}, and \texttt{inventory}, then merge consecutive identical actions. Object identities and order remain. We count each distinct sequence once per task, including the empty sequence. A sequence is successful if any occurrence succeeds and failure-only otherwise, even when filtering gives successful and failed trajectories the same sequence. Task averages include tasks with no successful trajectory.

Total sequence counts in Fig.~\ref{fig:all-sequence-dynamics} decrease after OPW, while successful counts in Fig.~\ref{fig:successful-sequence-dynamics} increase in both environments and model sizes. In ALFWorld 4B, the total falls from 5.18 to 4.31, successful sequences rise from 2.23 to 4.00, and failure-only sequences fall from 2.95 to 0.31.

\begin{figure}[t]
\centering
\includegraphics[width=\textwidth]{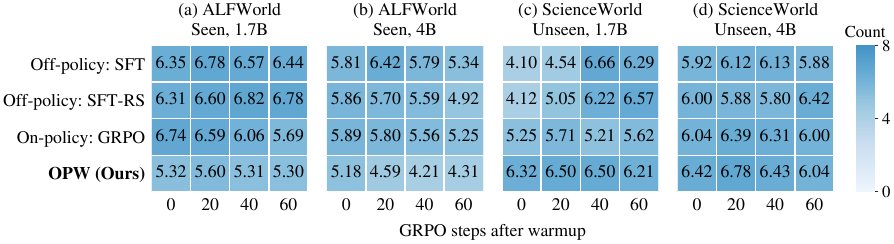}
\caption{All distinct action sequences among eight rollouts per task during GRPO.}
\label{fig:all-sequence-dynamics}
\end{figure}

To compare equal numbers of successes, we retain tasks with at least $m\in\{2,4\}$ successful trajectories under every method at all four evaluations. If a task has $S$ successes and sequence $j$ occurs $c_j$ times, its expected distinct count in $m$ successes selected without replacement is $\sum_j[1-\binom{S-c_j}{m}/\binom{S}{m}]$. Table~\ref{tab:conditional-success-sequences} averages over the $N$ retained tasks, at warmup end and step 60.

\begin{table}[htbp]
\centering\fontsize{8.5}{10}\selectfont
\caption{Distinct action sequences among equal numbers of successful trajectories.}
\label{tab:conditional-success-sequences}
\setlength{\tabcolsep}{2pt}
\renewcommand{\arraystretch}{1.12}
\begin{tabular}{@{\hspace{3pt}}>{\raggedright\arraybackslash}p{54pt}>{\centering\arraybackslash}p{22pt}>{\centering\arraybackslash}p{8pt}>{\centering\arraybackslash}p{12pt}*{3}{>{\centering\arraybackslash}p{55pt}}>{\columncolor{OPDLightBlue}\centering\arraybackslash}p{55pt}@{\hspace{3pt}}}
\toprule
\textbf{Environment} & \textbf{Size} & $m$ & $N$ & \shortstack{Off-policy:\\SFT} & \shortstack{Off-policy:\\SFT-RS} & \shortstack{On-policy:\\GRPO} & \textbf{OPW (Ours)} \\
\midrule
\textbf{ALFWorld} & 1.7B & 2 & 10 & $1.72\to1.64$ & $1.81\to1.54$ & $1.52\to1.48$ & $1.51\to1.34$ \\
 & & 4 & 1 & $3.00\to3.21$ & $2.00\to2.60$ & $1.97\to2.41$ & $1.00\to1.93$ \\
 & 4B & 2 & 49 & $1.69\to1.60$ & $1.70\to1.50$ & $1.62\to1.58$ & $1.67\to1.52$ \\
 & & 4 & 39 & $2.53\to2.46$ & $2.65\to2.08$ & $2.48\to2.36$ & $2.54\to2.14$ \\
\midrule
\textbf{ScienceWorld} & 1.7B & 2 & 1 & $1.67\to1.00$ & $1.73\to1.46$ & $1.68\to1.00$ & $1.29\to1.00$ \\
 & & 4 & 1 & $2.43\to1.00$ & $2.60\to2.00$ & $2.41\to1.00$ & $1.57\to1.00$ \\
 & 4B & 2 & 17 & $1.47\to1.57$ & $1.56\to1.59$ & $1.50\to1.60$ & $1.43\to1.57$ \\
 & & 4 & 7 & $1.98\to2.02$ & $2.03\to2.31$ & $1.52\to2.07$ & $1.83\to2.32$ \\
\bottomrule
\end{tabular}
\end{table}

For OPW 4B with $m=4$, the expected count falls from 2.54 to 2.14 in ALFWorld but rises from 1.83 to 2.32 in ScienceWorld. More successful sequences within eight rollouts can therefore accompany either less or more variation among equally many successes. All 1.7B cells except ALFWorld with $m=2$ describe a single task.

\section{Warmup Duration and Repeated Task Success}
\label{app:process-switching}

\subsection{What Continues to Improve during Warmup?}
\label{app:opd-switch-diagnostics}

Table~\ref{tab:opd-validation-progress} evaluates Qwen3-1.7B OPW with data-order seed 4 every five steps on 64 reserved ALFWorld analysis tasks, using 32 rollouts per task. These differ from the 140 Seen test tasks in the full-training comparison. Pass@8 and pass\textasciicircum{}8 use Eq.~\ref{eq:success-estimators} with $n=32$. From steps 20 to 40, validity gains slow while repeated task success continues to improve. Tasks solved at least once in 32 rollouts increase from 52 to 53, while those solved in every rollout increase from 2 to 11.

\begin{table}[htbp]
\caption{Action validity and repeated task success during OPW (ALFWorld 1.7B, $n=32$, \%).}
\label{tab:opd-validation-progress}
\par\smallskip
\centering\fontsize{8.5}{10}\selectfont
\setlength{\tabcolsep}{2pt}
\renewcommand{\arraystretch}{1.12}
\begin{tabular}{@{\hspace{3pt}}*{4}{>{\centering\arraybackslash}p{53pt}}@{\hspace{3pt}}}
\toprule
\textbf{OPW} steps & Valid actions & pass@8 & pass\textasciicircum{}8 \\
\midrule
10 & 72.9 & 39.4 & 0.4 \\
15 & 85.1 & 52.4 & 1.3 \\
\rowcolor{TableHeader}
20 & 89.4 & 69.7 & 10.8 \\
25 & 90.4 & 71.4 & 13.4 \\
30 & 91.1 & 71.2 & 14.5 \\
35 & 91.6 & 73.7 & 20.7 \\
\rowcolor{TableHeader}
40 & 91.8 & 74.8 & 22.0 \\
\bottomrule
\end{tabular}
\end{table}

For ScienceWorld Qwen3-1.7B with data-order seed 3, we evaluate 64 Seen tasks every five OPW steps from 10 to 40, with eight rollouts per task. These samples differ from initial GRPO evaluations. Fig.~\ref{fig:opd-process-action-value} compares partial-credit reward contrast with complete-task success. At both steps 15 and 40, 92.2\% of groups contain different scores, while pass@8 rises from 18.8\% to 56.2\%. Success coverage thus continues to improve when partial-credit reward differences are already common.

\begin{figure}[htbp]
\centering
\includegraphics[width=0.65\linewidth]{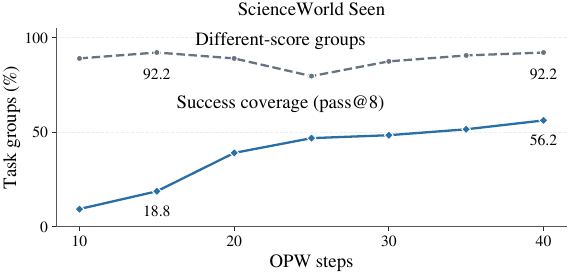}
\caption{Success coverage and partial-credit reward contrast during OPW (ScienceWorld 1.7B).}
\label{fig:opd-process-action-value}
\end{figure}

\subsection{Learning Speed and Repeated Task Success}
\label{app:duration-comparisons}
\label{app:duration-common-checkpoint}

Table~\ref{tab:duration-reliability} compares Qwen3-1.7B paths sharing a warmup run and ending at 100 total steps, evaluated on both benchmark splits with eight rollouts per task. ALFWorld branches after 20 or 40 OPW steps from the same seed-4 run. Longer warmup raises final Seen pass@1 from 88.3\% to 94.9\% and pass\textasciicircum{}8 from 69.3\% to 80.7\%. With evaluations every five steps, the first reaching 80\% success moves from total step 75 to 60. Total steps are fixed, but GPU-hours differ.

ScienceWorld branches from one seed-1 checkpoint after 40 OPW steps. Switching to GRPO immediately yields higher Unseen pass@1 (36.6\% versus 32.2\%) and pass\textasciicircum{}8 (4.5\% versus 3.5\%) than 20 more OPW steps. Mean score favors longer warmup (63.84 versus 63.06). Longer warmup improves Seen mean score (71.72 to 76.65) and pass@8 (84.4\% to 85.9\%). The preferred allocation depends on the task set and whether evaluation rewards partial progress or complete success.

\begin{table}[htbp]
\caption{Repeated task success after 100 total steps (Qwen3-1.7B, $n=8$, \%).}
\label{tab:duration-reliability}
\par\smallskip
\centering\fontsize{8.5}{10}\selectfont
\setlength{\tabcolsep}{2pt}
\renewcommand{\arraystretch}{1.12}
\begin{tabular}{@{\hspace{3pt}}>{\raggedright\arraybackslash}p{111pt}>{\centering\arraybackslash}p{46pt}*{4}{>{\centering\arraybackslash}p{42pt}}@{\hspace{3pt}}}
\toprule
\textbf{Environment / task set} & \textbf{OPW} steps & pass\textasciicircum{}1 & pass\textasciicircum{}2 & pass\textasciicircum{}4 & pass\textasciicircum{}8 \\
\midrule
\textbf{ALFWorld} / Seen & 20 & 88.3 & 83.5 & 77.2 & 69.3 \\
 & 40 & \textbf{94.9} & \textbf{91.8} & \textbf{86.9} & \textbf{80.7} \\
\midrule
\textbf{ScienceWorld} / Unseen & 40 & \textbf{36.6} & \textbf{20.1} & \textbf{8.7} & \textbf{4.5} \\
 & 60 & 32.2 & 17.8 & 8.5 & 3.5 \\
\midrule
\textbf{ScienceWorld} / Seen & 40 & 51.8 & 37.4 & 24.4 & 10.9 \\
 & 60 & \textbf{56.1} & \textbf{42.7} & \textbf{30.7} & \textbf{20.3} \\
\bottomrule
\end{tabular}
\end{table}

\noindent\textbf{Four-duration comparison across training runs.}
Fig.~\ref{fig:warmup-reliability} compares 20, 40, 60, and 80 OPW steps followed by GRPO to 100 steps, evaluated every 20 steps. ALFWorld uses seed 4 for 20/40, seed 3 for 60, and seed 1 for 80 OPW steps. ScienceWorld uses seed 3 for 20/40 and seed 1 for 60/80. Rankings include training-run differences. ScienceWorld's 40-step curve ends at 2.0\% Unseen pass\textasciicircum{}8, while the seed-1 branch in Table~\ref{tab:duration-reliability} ends at 4.5\%. The 60-step branch is shared.

Within the ScienceWorld seed-3 run, extending warmup from 20 to 40 steps raises Unseen pass@1 from 28.3\% to 34.1\% and pass\textasciicircum{}8 from 1.0\% to 2.0\%, while mean score changes little (60.94 to 60.64). Seen pass\textasciicircum{}8 also rises, from 3.1\% to 10.9\%. These comparisons support choosing duration by subsequent validation on the required outcome.

\section{Long-Trajectory Examples}
\label{app:trajectory-cases}

Two ALFWorld Unseen tasks illustrate long interactions and failure patterns. Comparisons use Qwen3-4B on the same task and the sixth evaluation trajectory. Base uses released weights, while SFT and OPW use their 40-step warmup models. Table~\ref{tab:warmup-action-composition} reports task-wide error frequencies.

The tables retain every action and quote selected environment responses, omitting reasoning and other observations. \colorbox{orange!15}{Orange} marks invalid actions, \colorbox{red!12}{red} marks repeated invalid actions, and R marks any repetition. Other actions are valid. Highlights and turn labels are our annotations.

\noindent\textbf{OPW completes the requested task while retaining useful repeated actions.}
Table~\ref{tab:case-pan} follows a task that requires cleaning a pan and placing it on a countertop. Base cleans and places a pot instead, then repeats \texttt{finish} without completing the task. OPW locates the pan and completes the task in 26 actions. Its returns to the sink and countertop serve cleaning and placement.

\begingroup
\footnotesize
\setlength{\tabcolsep}{3pt}
\renewcommand{\arraystretch}{1.06}
\setlength{\LTcapwidth}{\linewidth}
\begin{longtable}{@{}>{\columncolor{white}[0pt][\tabcolsep]\raggedleft\arraybackslash}p{12pt}p{\dimexpr.5\linewidth-39pt\relax}>{\columncolor{white}[\tabcolsep][0pt]\centering\arraybackslash}p{9pt}@{\hspace{12pt}}>{\columncolor{white}[0pt][\tabcolsep]\raggedleft\arraybackslash}p{12pt}p{\dimexpr.5\linewidth-39pt\relax}>{\columncolor{white}[\tabcolsep][0pt]\centering\arraybackslash}p{9pt}@{}}
\caption{\small Complete action sequences before and after OPW on the pan-cleaning task.}
\label{tab:case-pan} \\
\toprule
\multicolumn{6}{@{}p{\linewidth}@{}}{\textbf{Task:} ``clean some pan and put it in countertop.''\quad ALFWorld Unseen, Qwen3-4B.} \\
\multicolumn{3}{@{}>{\columncolor{TableHeader}[0pt][0pt]}c@{\hspace{12pt}}}{\textbf{Base}} & \multicolumn{3}{@{}>{\columncolor{OPDLightBlue}[0pt][0pt]}c@{}}{\textbf{OPW (Ours)}} \\
\midrule
\endfirsthead
\multicolumn{6}{c}{\tablename~\thetable\ (continued)} \\
\toprule
\multicolumn{6}{@{}p{\linewidth}@{}}{\textbf{Task:} ``clean some pan and put it in countertop.''\quad ALFWorld Unseen, Qwen3-4B.} \\
\multicolumn{3}{@{}>{\columncolor{TableHeader}[0pt][0pt]}c@{\hspace{12pt}}}{\textbf{Base}} & \multicolumn{3}{@{}>{\columncolor{OPDLightBlue}[0pt][0pt]}c@{}}{\textbf{OPW (Ours)}} \\
\midrule
\endhead
\bottomrule
\endfoot
\bottomrule
\endlastfoot
\caseaction{1}{go to cabinet 1}{} & \caseaction{1}{go to cabinet 1}{} \\
\caseaction{2}{go to cabinet 2}{} & \caseaction{2}{go to cabinet 2}{} \\
\caseaction{3}{open cabinet 2}{} & \caseaction{3}{open cabinet 2}{} \\
\caseaction{4}{go to cabinet 3}{} & \caseaction{4}{go to cabinet 3}{} \\
\caseaction{5}{go to cabinet 4}{} & \caseaction{5}{go to cabinet 4}{} \\
\caseaction{6}{go to cabinet 5}{} & \caseaction{6}{go to cabinet 5}{} \\
\caseaction{7}{open cabinet 5}{} & \caseaction{7}{open cabinet 5}{} \\
\caseaction{8}{go to cabinet 6}{} & \caseaction{8}{go to cabinet 6}{} \\
\caseaction{9}{open cabinet 6}{} & \caseaction{9}{open cabinet 6}{} \\
\caseaction{10}{go to drawer 1}{} & \caseaction{10}{go to drawer 1}{} \\
\caseaction{11}{open drawer 1}{} & \caseaction{11}{open drawer 1}{} \\
\caseaction{12}{go to drawer 2}{} & \caseaction{12}{go to drawer 2}{} \\
\caseaction{13}{open drawer 2}{} & \caseaction{13}{open drawer 2}{} \\
\caseaction{14}{go to drawer 3}{} & \caseaction{14}{go to drawer 3}{} \\
\caseaction{15}{open drawer 3}{} & \caseaction{15}{open drawer 3}{} \\
\caseaction{16}{go to shelf 1}{} & \caseaction{16}{go to countertop 1}{} \\
\caseaction{17}{go to shelf 2}{} & \caseaction{17}{go to countertop 2}{} \\
\caseaction{18}{go to shelf 3}{} & \caseaction{18}{go to countertop 3}{} \\
\caseaction{19}{go to sinkbasin 1}{} & \caseaction{19}{go to sinkbasin 1}{} \\
\caseaction{20}{go to stoveburner 1}{} & \caseaction{20}{go to stoveburner 1}{} \\
\caseaction{21}{take pot 1 from stoveburner 1}{} & \caseaction{21}{go to stoveburner 2}{} \\
\caseaction{22}{go to sinkbasin 1}{R} & \caseaction{22}{take pan 1 from stoveburner 2}{} \\
\caseaction{23}{clean pot 1 with sinkbasin 1}{} & \caseaction{23}{go to sinkbasin 1}{R} \\
\caseaction{24}{go to countertop 1}{} & \caseaction{24}{clean pan 1 with sinkbasin 1}{} \\
\caseaction{25}{move pot 1 to countertop 1}{} & \caseaction{25}{go to countertop 1}{R} \\
\caseaction[orange!15]{26}{finish}{} & \caseaction{26}{move pan 1 to countertop 1}{} \\
\caseaction[red!12]{27}{finish}{R} & \caseaction{}{}{} \\
\caseaction[red!12]{28}{finish}{R} & \caseaction{}{}{} \\
\caseaction[red!12]{29}{finish}{R} & \caseaction{}{}{} \\
\caseaction[red!12]{30}{finish}{R} & \caseaction{}{}{} \\
\midrule
\multicolumn{3}{@{}p{\dimexpr.5\linewidth-6pt\relax}@{\hspace{12pt}}}{%
\textbf{Environment feedback}
\par Turn 21: ``You pick up the pot 1 from the stoveburner 1.''
\par Turns 26--30: ``Nothing happens.''
\par\smallskip
\textbf{Outcome:} Task not completed.\par 25/30 actions valid, 4 repeated invalid actions.
} & \multicolumn{3}{@{}p{\dimexpr.5\linewidth-6pt\relax}@{}}{%
\textbf{Environment feedback}
\par Turn 22: ``You pick up the pan 1 from the stoveburner 2.''
\par Turn 24: ``You clean the pan 1 using the sinkbasin 1.''
\par Turn 26: ``You move the pan 1 to the countertop 1.''
\par\smallskip
\textbf{Outcome:} Task completed.\par 26/26 actions valid, 0 repeated invalid actions.
} \\
\end{longtable}
\endgroup

\noindent\textbf{OPW completes the task requirements after a longer search.}
In Table~\ref{tab:case-egg}, SFT finds the egg earlier and places it in the microwave without cleaning it. It issues \texttt{end} 11 times with feedback ``Nothing happens.'' OPW searches longer but cleans and places the egg in 29 actions.

\begingroup
\footnotesize
\setlength{\tabcolsep}{3pt}
\renewcommand{\arraystretch}{1.06}
\setlength{\LTcapwidth}{\linewidth}
\begin{longtable}{@{}>{\columncolor{white}[0pt][\tabcolsep]\raggedleft\arraybackslash}p{12pt}p{\dimexpr.5\linewidth-39pt\relax}>{\columncolor{white}[\tabcolsep][0pt]\centering\arraybackslash}p{9pt}@{\hspace{12pt}}>{\columncolor{white}[0pt][\tabcolsep]\raggedleft\arraybackslash}p{12pt}p{\dimexpr.5\linewidth-39pt\relax}>{\columncolor{white}[\tabcolsep][0pt]\centering\arraybackslash}p{9pt}@{}}
\caption{\small Complete action sequences after SFT and OPW on the egg-cleaning task.}
\label{tab:case-egg} \\
\toprule
\multicolumn{6}{@{}p{\linewidth}@{}}{\textbf{Task:} ``put a clean egg in microwave.''\quad ALFWorld Unseen, Qwen3-4B.} \\
\multicolumn{3}{@{}>{\columncolor{TableHeader}[0pt][0pt]}c@{\hspace{12pt}}}{\textbf{Off-policy: SFT}} & \multicolumn{3}{@{}>{\columncolor{OPDLightBlue}[0pt][0pt]}c@{}}{\textbf{OPW (Ours)}} \\
\midrule
\endfirsthead
\multicolumn{6}{c}{\tablename~\thetable\ (continued)} \\
\toprule
\multicolumn{6}{@{}p{\linewidth}@{}}{\textbf{Task:} ``put a clean egg in microwave.''\quad ALFWorld Unseen, Qwen3-4B.} \\
\multicolumn{3}{@{}>{\columncolor{TableHeader}[0pt][0pt]}c@{\hspace{12pt}}}{\textbf{Off-policy: SFT}} & \multicolumn{3}{@{}>{\columncolor{OPDLightBlue}[0pt][0pt]}c@{}}{\textbf{OPW (Ours)}} \\
\midrule
\endhead
\bottomrule
\endfoot
\bottomrule
\endlastfoot
\caseaction{1}{go to fridge 1}{} & \caseaction{1}{go to fridge 1}{} \\
\caseaction{2}{open fridge 1}{} & \caseaction{2}{open fridge 1}{} \\
\caseaction{3}{go to cabinet 1}{} & \caseaction{3}{go to cabinet 1}{} \\
\caseaction{4}{go to cabinet 2}{} & \caseaction{4}{go to cabinet 2}{} \\
\caseaction{5}{open cabinet 2}{} & \caseaction{5}{open cabinet 2}{} \\
\caseaction{6}{go to drawer 1}{} & \caseaction{6}{go to cabinet 3}{} \\
\caseaction{7}{open drawer 1}{} & \caseaction{7}{go to cabinet 4}{} \\
\caseaction{8}{go to shelf 1}{} & \caseaction{8}{go to cabinet 5}{} \\
\caseaction{9}{go to shelf 2}{} & \caseaction{9}{open cabinet 5}{} \\
\caseaction{10}{go to shelf 3}{} & \caseaction{10}{go to cabinet 6}{} \\
\caseaction{11}{go to countertop 1}{} & \caseaction{11}{open cabinet 6}{} \\
\caseaction{12}{go to countertop 2}{} & \caseaction{12}{go to drawer 1}{} \\
\caseaction{13}{go to countertop 3}{} & \caseaction{13}{open drawer 1}{} \\
\caseaction{14}{take egg 1 from countertop 3}{} & \caseaction{14}{go to drawer 2}{} \\
\caseaction{15}{examine egg 1}{} & \caseaction{15}{open drawer 2}{} \\
\caseaction{16}{go to microwave 1}{} & \caseaction{16}{go to drawer 3}{} \\
\caseaction{17}{open microwave 1}{} & \caseaction{17}{open drawer 3}{} \\
\caseaction{18}{move egg 1 to microwave 1}{} & \caseaction{18}{go to shelf 1}{} \\
\caseaction{19}{close microwave 1}{} & \caseaction{19}{go to shelf 2}{} \\
\caseaction[orange!15]{20}{end}{} & \caseaction{20}{go to shelf 3}{} \\
\caseaction[red!12]{21}{end}{R} & \caseaction{21}{go to countertop 1}{} \\
\caseaction[red!12]{22}{end}{R} & \caseaction{22}{go to countertop 2}{} \\
\caseaction[red!12]{23}{end}{R} & \caseaction{23}{go to countertop 3}{} \\
\caseaction[red!12]{24}{end}{R} & \caseaction{24}{take egg 1 from countertop 3}{} \\
\caseaction[red!12]{25}{end}{R} & \caseaction{25}{go to sinkbasin 1}{} \\
\caseaction[red!12]{26}{end}{R} & \caseaction{26}{clean egg 1 with sinkbasin 1}{} \\
\caseaction[red!12]{27}{end}{R} & \caseaction{27}{go to microwave 1}{} \\
\caseaction[red!12]{28}{end}{R} & \caseaction{28}{open microwave 1}{} \\
\caseaction[red!12]{29}{end}{R} & \caseaction{29}{move egg 1 to microwave 1}{} \\
\caseaction[red!12]{30}{end}{R} & \caseaction{}{}{} \\
\midrule
\multicolumn{3}{@{}p{\dimexpr.5\linewidth-6pt\relax}@{\hspace{12pt}}}{%
\textbf{Environment feedback}
\par Turn 18: ``You move the egg 1 to the microwave 1.''
\par Turns 20--30: ``Nothing happens.''
\par\smallskip
\textbf{Outcome:} Task not completed.\par 19/30 actions valid, 10 repeated invalid actions.
} & \multicolumn{3}{@{}p{\dimexpr.5\linewidth-6pt\relax}@{}}{%
\textbf{Environment feedback}
\par Turn 26: ``You clean the egg 1 using the sinkbasin 1.''
\par Turn 29: ``You move the egg 1 to the microwave 1.''
\par\smallskip
\textbf{Outcome:} Task completed.\par 29/29 actions valid, 0 repeated invalid actions.
} \\
\end{longtable}
\endgroup

\section{Limitations and Scope}
\label{app:limitations}

Our experiments use ALFWorld and ScienceWorld as controlled testbeds for studying how OPW affects subsequent RLVR in multi-step interactive tasks. Extending this investigation to more diverse interaction settings is a natural direction for future work. Our theoretical analysis characterizes initial reward discovery under the stated assumptions, complementing the empirical evaluation of subsequent learning dynamics. In practice, the appropriate warmup duration and end-to-end training efficiency depend on the task, model configuration, and cost of teacher supervision. We examine multiple warmup durations and include teacher scoring in the recorded training cost, distinguishing faster subsequent learning from overall computational savings. Adaptive criteria for transitioning from warmup to RLVR could further improve the allocation of training resources.

\end{document}